\documentclass[letterpaper]{article} 
\usepackage{aaai2027}  
\usepackage[hyphens]{url}  
\usepackage{graphicx} 
\usepackage{natbib}  
\usepackage{caption} 
\usepackage{algorithm}
\usepackage{algpseudocode}
\usepackage{newfloat}
\usepackage{listings}
\DeclareCaptionStyle{ruled}{labelfont=normalfont,labelsep=colon,strut=off} 
\floatstyle{ruled}
\newfloat{listing}{tb}{lst}{}
\floatname{listing}{Listing}

\usepackage{booktabs}

\usepackage{booktabs}       
\usepackage{amsfonts}       
\usepackage{nicefrac}       
\usepackage{microtype}      
\usepackage{xcolor}         
\usepackage{graphicx}
\usepackage{amsmath}
\usepackage{amsthm}

\usepackage{booktabs}
\usepackage{multirow}
\usepackage[table]{xcolor}
\usepackage{graphicx}

\usepackage[table]{xcolor}

\definecolor{lightgraycell}{RGB}{242,242,242}
\definecolor{lightbluecell}{RGB}{225,240,255}

\usepackage{algorithm}
\usepackage{algpseudocode}

\usepackage{pifont}

\usepackage{colortbl}

\usepackage{subcaption}

\newtheorem{theorem}{Theorem}
\newtheorem{proposition}[theorem]{Proposition}
\newtheorem{lemma}[theorem]{Lemma}
\newtheorem{corollary}[theorem]{Corollary}
\newtheorem{assumption}{Assumption}
\newtheorem{definition}{Definition}
\newtheorem{remark}{Remark}

\usepackage{pifont}

\nocopyright

\title{\texttt{FREPix}: Frequency-Heterogeneous Flow Matching for Pixel-Space Image Generation}
\author{
    Mingfeng Lin, Jiakun Chen, Liang Han\corresponding,
    Liqiang Nie\corresponding
}
\affiliations{
    Harbin Institute of Technology (Shenzhen)\\

}

\begin{document}

\maketitle

\begin{abstract}
Pixel-space diffusion has re-emerged as a promising alternative to latent-space generation because it avoids the representation bottleneck introduced by VAEs. Yet most existing methods still treat image generation as a frequency-homogeneous process, overlooking the distinct roles and learning dynamics of low- and high-frequency components. To address this, we propose \texttt{FREPix}, a \textbf{FRE}quency-heterogeneous flow matching framework for \textbf{Pix}el-space image generation. \texttt{FREPix} explicitly decomposes generation into low- and high-frequency components, assigns them separate transport paths, predicts them with a factorized network, and trains them with a frequency-aware objective. In this way, coarse-to-fine generation becomes an explicit design principle rather than an implicit behavior. On ImageNet class-to-image generation, \texttt{FREPix} achieves competitive results among pixel-space generation models, reaching 1.91 FID at $256\times256$ and 2.38 FID at $512\times512$, with particularly strong performance in the early stages of training and in the low-NFE regime.

\end{abstract}

\section{Introduction}
Latent diffusion~\cite{dhariwal2021diffusion,rombach2022high,peebles2023scalable,ma2024sit,leng2025repa} has become the dominant paradigm for image generation by moving denoising from raw pixels to a compact latent space, which greatly reduces spatial complexity and makes large-scale training practical. But this efficiency comes with a structural cost. Generation is no longer performed in the original image domain, and image quality is inevitably tied to the representation and reconstruction fidelity of the VAEs~\cite{yao2025reconstruction,shi2025latent}. These limitations have renewed interest in pixel-space generation, where models operate directly on raw images and avoid the representational bottleneck introduced by latent space encodings.

Despite this appeal, pixel-space generation remains fundamentally difficult. Raw images are high-dimensional, spatially dense and entangle global semantics with local details in a single state space~\cite{rahaman2019spectral}. Recent progress has made this paradigm increasingly viable through coarse-to-fine architectures~\cite{ho2022cascaded,teng2024relay} and stronger pixel-level modeling~\cite{yu2025pixeldit,wang2025pixnerd}. Still, most existing methods treat image generation as a homogeneous process. They model the whole image with a single state and leave the separation between global structures and fine details to emerge implicitly during training.

\begin{figure}
  \centering
  \includegraphics[width=0.45\textwidth]{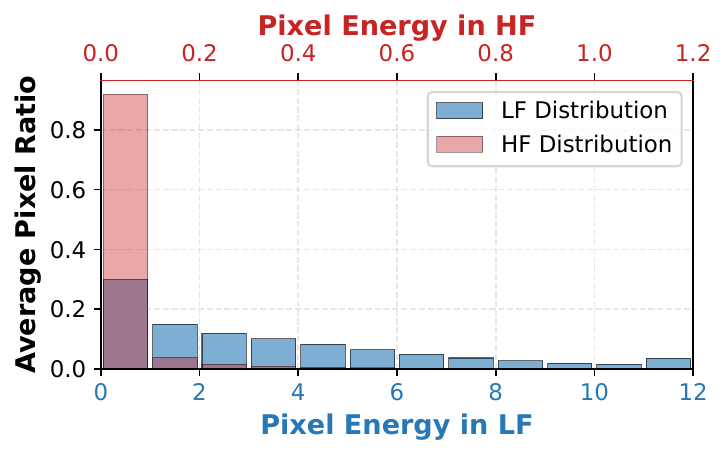}
  \caption{\textbf{Frequency heterogeneity in natural images.} The low-frequency component exhibits larger per-location energy and a broader distribution than the high-frequency one. The energy is measured by the squared $\ell_2$ norm of the corresponding low-/high-frequency coefficients at each location.
  }
  \label{fig:1}
\end{figure}

Natural images are not organized uniformly across frequencies. Low-frequency components mainly determine large-scale layout, color composition, and semantic structure, while high-frequency components are more closely associated with edges, textures, and perceptual sharpness~\cite{torralba2003statistics,chen2019drop}. More importantly, these two differ not only in visual role, but also in their underlying statistics and learning dynamics~\cite{xu2020frequency}. As illustrated in Fig. \ref{fig:1}, the stark divergence in their energy distributions provides an empirical evidence of this heterogeneity. Therefore, treating these heterogeneous components with a single state representation, a shared interpolation path, and a unified modeling strategy imposes an unnecessarily restrictive inductive bias on pixel-space generation.

In this paper, we propose a \textbf{FRE}quency-heterogeneous flow matching framework for \textbf{Pix}el-space image generation (\texttt{FREPix}). Natural images are frequency-heterogeneous, yet current pixel-space generation is still formulated largely as a frequency-homogeneous process. \texttt{FREPix} makes this heterogeneity explicit throughout generation. It decomposes the image into low- and high-frequency sub-states, assigns them heterogeneous interpolation paths, and explicitly decomposes the generation target: a low-frequency backbone first predicts the clean low-frequency component, while a high-frequency decoder then predicts the corresponding high-frequency component conditioned on the predicted low-frequency component. Training objective is further aligned with this factorization through a specifically designed frequency-aware flow matching objective. In this way, \texttt{FREPix} turns coarse-to-fine generation from an implicit behavior that the network is expected to discover into an explicit design principle for pixel-space flow matching. 

Extensive experiments validate the effectiveness of \texttt{FREPix} in class-to-image generation.  It achieves competitive results among pixel-space generation models on ImageNet, reaching 1.91 FID at 256×256 and 2.38 FID at 512×512, while also attaining competitive quality at an early stage of training and under low-NFE sampling. Together, these results show that explicitly modeling frequency heterogeneity provides a stronger inductive bias for end-to-end pixel-space generation.

\section{Related Work}

\paragraph{Latent-Space and Pixel-Space Image Generation.}
Modern image generation methods can be broadly divided into latent-space and pixel-space approaches. Latent diffusion models~\cite{rombach2022high} improve efficiency by performing denoising in a compressed representation, and have been further strengthened by transformer-based formulations such as DiT~\cite{peebles2023scalable} and SiT~\cite{ma2024sit}. However, their performance is inherently tied to the quality of the autoencoder, which may limit fine-detail reconstruction and introduce decoding artifacts. This has renewed interest in direct pixel-space generation~\cite{ho2020denoising,nichol2021improved}, where recent models such as JiT~\cite{li2025back}, PixelDiT~\cite{yu2025pixeldit}, PixNerd~\cite{wang2025pixnerd}, DeCo~\cite{ma2025deco} and PixelGen~\cite{ma2026pixelgen} show that raw-pixel modeling can achieve competitive image quality. Most existing pixel-space methods, however, still model the image as a largely homogeneous state or introduce frequency cues mainly through architecture or auxiliary losses, without explicitly defining separate transport paths for different frequency components.

\paragraph{Frequency-Aware and Coarse-to-Fine Image Generation.}
A long line of work has explored the observation that images are naturally generated in a coarse-to-fine manner. Cascaded diffusion models~\cite{ho2022cascaded} decompose generation across resolutions, while multi-scale pixel-space methods such as SiD2~\cite{hoogeboom2025simpler} and PixelFlow~\cite{chen2025pixelflow} reduce the difficulty of raw-pixel modeling through structured resolution schedules. Frequency-aware methods further exploit the distinction between frequency. For example, WDM~\cite{phung2023wavelet} applies wavelet-domain modeling, PixelDiT~\cite{yu2025pixeldit} separates global and local information through a dual-level design, and DeCo~\cite{ma2025deco} combines a low-frequency backbone with high-frequency refinement. EqualSNR~\cite{falck2025fourier} study how training signals vary across frequencies, whereas our method focuses on defining frequency-specific flow-matching paths for low- and high-frequency sub-states. In contrast to prior frequency-aware architectures or losses, \texttt{FREPix} aligns the transport path, prediction target, architecture, and objective with the coarse-to-fine structure of images.

\section{Frequency-Decoupled Flow Matching}
\subsection{Frequency-Decomposed State Space and Heterogeneous Interpolation}
\label{sec:state_space}

Standard pixel-space flow matching methods~\cite{li2025back,ma2025deco} represent the sample at time $t$ by a single state $x_t \in \mathbb{R}^d$ and, under the standard linear path, apply the same interpolation schedule to all image components through a shared vector field. While convenient, this homogeneous formulation does not explicitly reflect the frequency heterogeneity of natural images.

\paragraph{Frequency-decomposed state space.}
To make this heterogeneity explicit without sacrificing exactness, we reparameterize the image state with an orthonormal discrete wavelet transform (DWT) $\mathcal{W}:\mathbb{R}^d \rightarrow \mathbb{R}^d$. For any sample $x_t$, we write
\begin{equation}
\footnotesize
\label{eq:state_factorization_rewrite}
(l_t, h_t) = \mathcal{W}(x_t), 
\qquad
x_t = \mathcal{W}^{-1}(l_t, h_t),
\end{equation}
where $l_t \in \mathbb{R}^{d_l}$ denotes the low-frequency sub-state \emph{(structure)} and $h_t \in \mathbb{R}^{d_h}$ denotes the high-frequency sub-state \emph{(detail)}, with $d_l + d_h = d$. Since $\mathcal{W}$ is orthonormal, the factorization is exact and preserves signal energy by Parseval's identity~\cite{percival1995estimation}. Thus, unlike latent compression, this frequency factorization is lossless and changes only the parameterization of the state space.

\begin{figure}[!t]
    \centering
    \includegraphics[width=0.9\linewidth]{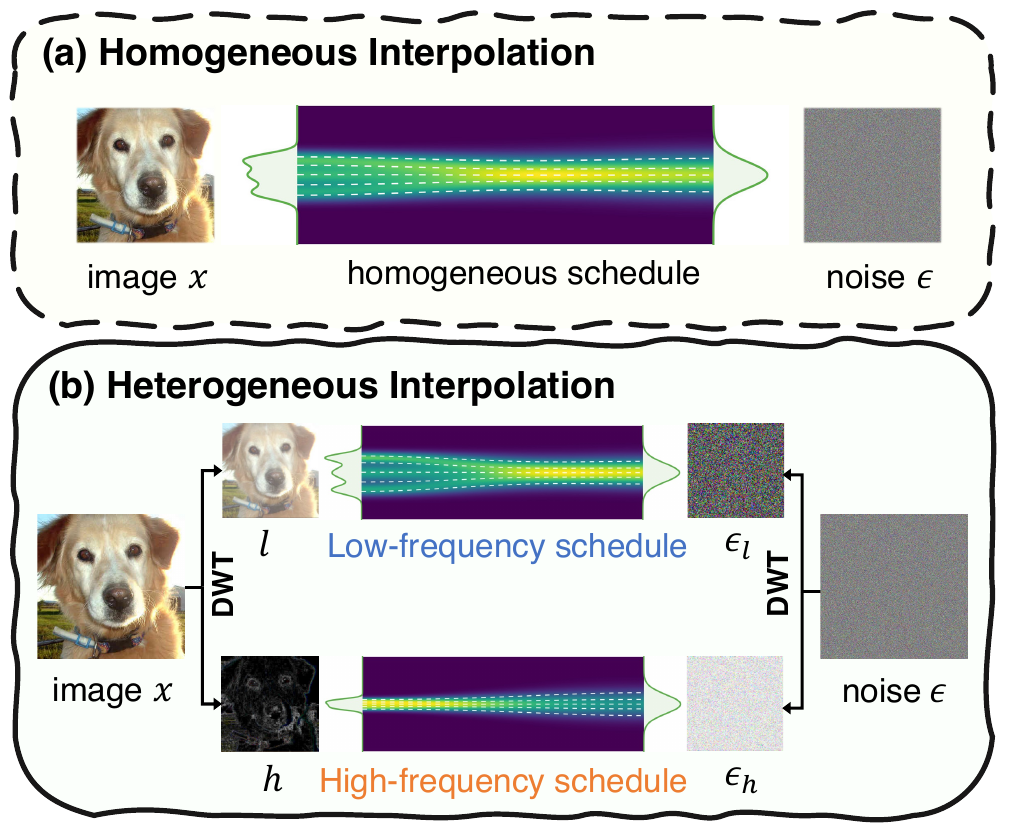}
    \caption{\textbf{Homogeneous vs. heterogeneous interpolation.} Standard pixel-space flow matching applies a shared interpolation schedule to all frequency components, treating the image as a homogeneous state. In contrast, our method first decomposes the image into low- and high-frequency sub-states and then assigns them separate schedules $g_l(t)$ and $g_h(t)$.}
    \label{fig:path}
\end{figure}

\paragraph{From decomposed states to heterogeneous interpolation.}
Once the image state is explicitly decomposed into low- and high-frequency components, it is natural to decompose the interpolation path accordingly rather than transport all frequencies with a single shared schedule. Let $x \sim \rho_1$ be a clean image and $\epsilon \sim \rho_0=\mathcal{N}(0,I_d)$ be a source noise, with $(l,h)=\mathcal{W}(x)$ and $(\epsilon_l,\epsilon_h)=\mathcal{W}(\epsilon)$. We define the heterogeneous interpolation path by
\begin{equation}
\label{eq:hetero_interp_rewrite}
\footnotesize
l_t = g_l(t)\,l + \bigl(1-g_l(t)\bigr)\epsilon_l,
\quad
h_t = g_h(t)\,h + \bigl(1-g_h(t)\bigr)\epsilon_h,
\end{equation}
where $g_l,g_h\in C^1([0,1])$ are strictly increasing schedules that satisfy $g_l(0)=g_h(0)=0$ and $g_l(1)=g_h(1)=1$. This allows the two frequency sub-states to follow different transport dynamics.

For notational convenience, we further write the path in operator form as
\begin{equation}
\footnotesize
\label{eq:G_operator_rewrite}
\begin{aligned}
x_t
&=
G(t)x
+
\bigl(I-G(t)\bigr)\epsilon,
\quad 
\dot{x}_t
=
\dot{G}(t)(x-\epsilon)
\\
G(t)
&=
\mathcal{W}^{-1}
\left(
\begin{smallmatrix}
g_l(t)I_{d_l} & 0
\\
0 & g_h(t)I_{d_h}
\end{smallmatrix}
\right)
\mathcal{W}.
\end{aligned}
\end{equation}
where $x_t$ denotes the pixel-space state and $\dot{x}_t$ is its time derivative, i.e., its conditional velocity. This formulation preserves linear operator interpolation between data and noise in pixel space, while generalizing the homogeneous scalar schedule of standard flow matching to a frequency-aware operator over explicitly decomposed sub-states. Fig.~\ref{fig:path} further illustrates the difference.

\paragraph{Why use different schedules?}
Different schedules $g_l$ and $g_h$ let the transport process reflect the frequency heterogeneity of natural images. Under standard pixel-space flow matching, all frequency components follow the same schedule. Once the state space is decomposed, this homogeneous design becomes unnecessarily restrictive. We therefore adopt frequency-heterogeneous interpolation, allowing low- and high-frequency sub-states to evolve under different transport dynamics within a unified framework. We studies the empirical instantiation of schedules in the experiments.

\begin{proposition}[Validity of Heterogeneous Interpolation]
\label{thm:validity_rewrite}
Assume $\mathcal{W}$ is orthonormal, $x \sim \rho_1$ has finite second moment, $\epsilon \sim \mathcal{N}(0,I_d)$, and $g_l, g_h \in C^1([0,1])$ are strictly increasing with $g_l(0)=g_h(0)=0$ and $g_l(1)=g_h(1)=1$. Let $x_t$ be defined by Eq.~\eqref{eq:G_operator_rewrite} and $\mathcal{B}$ denote the class of measurable vector fields $b(t,\cdot)$ such that $\int_0^1 \mathbb{E}\|b(t,x_t)\|^2\,dt < \infty.$ Then:

\begin{enumerate}
    \item \textbf{Smoothness:} The trajectory $t \mapsto x_t$ is almost surely continuously differentiable, with $\|\dot{x}_t\| \leq L_g(\|x\| + \|\epsilon\|)$;
    \item  \textbf{Continuity Equation:} For every $t\in[0,1)$, the law of $x_t$ admits a density $p_t$, and the marginal path satisfies the continuity equation $\partial_t p_t + \nabla \cdot (v_t p_t)=0$ in the sense of distribution, where $v_t(x_t)=\mathbb{E}[\dot{x}_t|x_t]$ is the marginal velocity field;
    \item \textbf{Learnability:} The population regression objective $
    \mathcal{L}(b) = \int_0^1 \mathbb{E}\!\left[\|b(t,x_t)-\dot x_t\|^2\right]dt$ is uniquely minimized (up to almost-everywhere equality) by the marginal velocity field $b^*(t,x_t)=\mathbb{E}[\dot x_t|x_t]=v_t(x_t)$.
\end{enumerate}
\end{proposition}

Proofs are in Appendix~C. Proposition~\ref{thm:validity_rewrite} shows that the heterogeneous interpolation is a principled extension of the standard flow matching path rather than a heuristic modification. This result is central to the remainder of our method: it justifies transport in the decomposed  $(l_t,h_t)$ state space and motivates the following network and the objective designs.

\begin{figure*}[!t]
    \centering
    \includegraphics[width=0.9\textwidth]{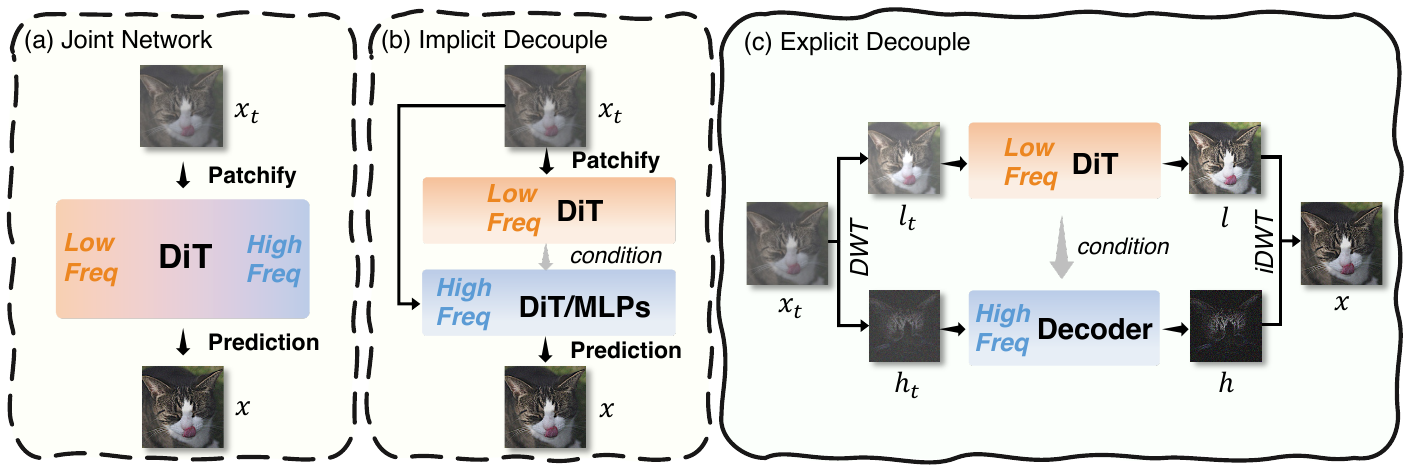}
    \caption{\textbf{Comparison of pixel-space methods.}
    (a) \textbf{Joint} (e.g., JiT) treats the image as a homogeneous state and predicts the clean target in one shot, leaving structure and detail entangled. (b) \textbf{Implicit decoupling} (e.g., DeCo) introduces staged pathways that can encourage specialization across scales, but does not explicitly assign frequency-specific prediction targets to different modules. (c) \textbf{Explicit decoupling} (Ours) directly factorizes prediction responsibility with a structure predictor predicting the clean low-frequency component and a detail refiner predicting the clean high-frequency component.}
    \label{fig:architecture}
\end{figure*}

\subsection{Factorized Generative Modeling via Explicit Architectural Decoupling}
\label{sec:architecture}

\paragraph{From decomposed transport to factorized generation.}
While we decomposes the state space and transport path by frequency, the generator should preserve this structure rather than collapse it back into a unified prediction problem. Fig.~\ref{fig:architecture} contrasts three architectural paradigms. In a joint design, one network operates on the mixed state $x_t$ and predicts the clean target in one shot, leaving the separation between low-frequency structure and high-frequency detail entirely implicit. More recent designs, such as DeCo and PixelDiT, introduce staged pathways that can encourage specialization across scales. However, this specialization is defined primarily through architectural organization and feature routing, rather than by explicitly specifying which module should predict which frequency component.

To avoid collapsing the decomposed transport back and make the decomposition explicit at the prediction level, we design the generator to model the heterogeneous sub-state $(l_t,h_t)$. Following JiT~\cite{li2025back}, we adopt \emph{$x$-prediction} parameterization. Specifically, we decouple the generation into two specialized modules: a \emph{structure predictor} $f_\varphi$ and a \emph{detail refiner} $g_\phi$:
\begin{equation}
\footnotesize
\label{eq:factorized_network}
\hat{l} = f_\varphi(l_t,t), \quad
\hat{h} = g_\phi(h_t,\hat{l},t), \quad
\hat{x} = \mathcal{W}^{-1}(\hat{l},\hat{h}).
\end{equation}

The structure predictor $f_\varphi$ is implemented as a Diffusion Transformer (DiT), which takes the noisy low-frequency sub-state $l_t$ as input and predicts the clean low-frequency component $\hat{l}$, thereby capturing long-range dependencies and global structure. To enable efficient high-frequency modeling without the computational overhead of self-attention, the high-frequency predictor $g_\phi$ is implemented as the Decoder from DeCo. It takes the noisy high-frequency sub-state $h_t$ as input to generate the clean high-frequency component $\hat{h}$ while using the predicted low-frequency structure $\hat{l}$ from the DiT as an explicit condition through AdaLN-Zero~\cite{peebles2023scalable}.

\paragraph{Implicit vs. explicit decoupling.}
The key distinction of our architecture lies not only in modularization alone, but also in how the decomposition is specified. In implicitly decoupled designs, staged pathways can encourage different modules to specialize in different frequency roles, but this specialization remains emergent from the architecture. Our model instead makes the decomposition explicit at the prediction level by assigning different prediction targets to different modules and enforcing a low-frequency to high-frequency conditional dependency between them. This turns coarse-to-fine generation from an architectural tendency into a hard design principle, aligning the network itself with the decomposed transport.

To make this distinction more concrete, let the direct joint function class be $\mathcal{F}_{\mathrm{dir}}$ and the explicit decoupled function class be $\mathcal{F}_{\mathrm{dec}}$. We analyze a simplified statistical setting in which clean targets and predictions are bounded, the clean low-frequency component concentrates near a $k_L$-dimensional manifold, and the relevant loss classes admit covering-number growth in~\cite{pollard1990empirical}. Proofs are in Appendix~C.

\begin{proposition}[Generalization comparison for explicit decoupling under simplified assumptions]
\label{prop:decoupling_complexity}
Let the ambient dimension be $d:=d_l+d_h$, and let $R_{\mathrm{dir}},\widehat{R}_{\mathrm{dir}}$ and
$R_{\mathrm{dec}},\widehat{R}_{\mathrm{dec}}$ denote the corresponding true and empirical risks, respectively. The following bounds hold simultaneously for all $f\in\mathcal{F}_{\mathrm{dir}}$ and
$(g_l,g_h)\in\mathcal{F}_{\mathrm{dec}}$ with probability at least $1-\delta$:
\begin{equation}
\footnotesize
\begin{aligned}
R_{\mathrm{dir}}(f)
&\le
\widehat{R}_{\mathrm{dir}}(f)
+
\frac{24A\sqrt{\pi d}}{\sqrt{N}}
+
3\sqrt{\frac{2\log(8/\delta)}{N}},
\\[0.5em]
R_{\mathrm{dec}}(g_L,g_H)
&\le
\widehat{R}_{\mathrm{dec}}(g_l,g_h)
+
\frac{12A\sqrt{\pi}}{\sqrt{N}}
\bigl(\sqrt{d_l}+\sqrt{k_l+d_h}\bigr)
\\
&\quad
+
3\sqrt{\frac{2\log(8/\delta)}{N}}.
\label{eq:bound_dec_main}
\end{aligned}
\end{equation}

where $k_l<d_l$ is the intrinsic dimension of the clean low-frequency component. Consequently, since $d=d_l+d_h$ and $k_l<d_l$, the decoupled complexity term is \textbf{smaller} than the corresponding direct-model term.
\end{proposition}

Proposition~\ref{prop:decoupling_complexity} should be interpreted as a comparison under a simplified statistical model rather than an exact characterization of modern Transformer-based architectures. Its role is to formalize the intuition that explicit decoupling can mitigate frequency entanglement by reducing the effective dimension and statistical complexity seen by each branch.

\subsection{Frequency-Aligned Flow Matching Objective}
\label{sec:objective}

With state, transport, and architecture all decomposed by frequency, the remaining question is how to align training with the same structure. In particular, the generator adopts an $x$-prediction parameterization, producing the clean reconstruction $x_\theta(x_t,t) = \text{net}_\theta(x_t, t)$ rather than regressing the velocity field directly. Following JiT, we preserve the optimization advantages of clean-data prediction, while recovering a flow matching training signal by analytically converting $x_\theta$ into the velocity $v_\theta$ induced by our heterogeneous path.

\paragraph{From $x$-prediction to induced velocity.}
Under the heterogeneous interpolation $x_t = G(t)x + \bigl(I-G(t)\bigr)\epsilon$, the conditional velocity induced by the path is obtained by differentiating with respect to $t$, which gives $
v_t(x_t|x)=\dot{G}(t)(x-\epsilon)$. Further, we can rewrite $v_t$ in terms of $x_t$ which gives
\begin{equation}
\footnotesize
v_t(x_t|x) = \dot{G}(t)\bigl(I-G(t)\bigr)^{-1}(x-x_t).
\end{equation}

This makes it possible to convert clean-image prediction into velocity prediction. Specifically, by replacing the clean target $x$ with its network prediction $\hat{x}$, we define the predicted velocity as
\begin{equation}
\footnotesize
v_\theta(t,x_t)=\dot{G}(t)\bigl(I-G(t)\bigr)^{-1}(\hat{x}-x_t).
\end{equation}

\paragraph{Frequency-aligned objective.}

Let $\mathcal{W}_l$ and $\mathcal{W}_h$ denote the low- and high-frequency projection operators of DWT where $(l,h)=( \mathcal{W}_l x, \mathcal{W}_h x) =  \mathcal{W}(x)$ and $\mathcal{W}^{-1}(l,h) = \mathcal{W}_l^T l + \mathcal{W}_h^T h$, the conditional velocity induced by the heterogeneous interpolation can be naturally decomposed into low- and high-frequency components:
\begin{equation}
\footnotesize
v_t(x_t|x) = \mathcal{W}_l^\top \dot g_l(t)(l-\epsilon_l) + \mathcal{W}_h^\top \dot g_h(t)(h-\epsilon_h).
\end{equation}

This decomposition allows us to explicitly control the relative difficulty of low- and high-frequency learning during training. Let $\lambda_l(t),\lambda_h(t)>0$ be time-dependent weights, we define the \emph{frequency-aligned conditional flow matching objective} as
\begin{equation}
\footnotesize
\label{eq:weighted_loss}
\begin{aligned}
\mathcal{L}(\theta)
:=
\mathbb{E}_{t,x,\epsilon}
\Big[
& \lambda_l(t)
\left\|
\mathcal{W}_l
\big(
v_\theta(t,x_t)
-
v_t(x_t\mid x)
\big)
\right\|^2
\\
&+
\lambda_h(t)
\left\|
\mathcal{W}_h
\big(
v_\theta(t,x_t)
-
v_t(x_t\mid x)
\big)
\right\|^2
\Big].
\end{aligned}
\end{equation}

\paragraph{Frequency weighting preserves the target flow.}
The weights $\lambda_l(t)$ and $\lambda_h(t)$ allow us to rebalance optimization between the low- and high-frequency components over time. This gives a simple mechanism for rebalancing optimization across frequency components. We defer the discussion of specific weighting choices in the experiments. Importantly, this reweighting should improve training dynamics without changing the target flow field. To make this explicit, define the time-dependent weighting matrix
\begin{equation}
\footnotesize
\label{eq:weight_matrix}
\mathbf{M}(t)
:=
\lambda_l(t)\mathcal{W}_l^\top\mathcal{W}_l
+
\lambda_h(t)\mathcal{W}_h^\top\mathcal{W}_h.
\end{equation}

Since $\mathcal{W}$ is orthonormal and $\lambda_l(t),\lambda_h(t)>0$, $\mathbf{M}(t)$ is positive definite for all $t$, and the objective in Eq.~\eqref{eq:weighted_loss} can be rewritten as
\begin{equation}
\footnotesize
\label{eq:mahalanobis}
\mathcal{L}(\theta)
=
\mathbb{E}_{t,x}
\Big[
\Delta v_\theta(t,x_t)^\top
\mathbf{M}(t)
\Delta v_\theta(t,x_t)
\Big].
\end{equation}
where $\Delta v_\theta(t,x_t)=v_\theta(t,x_t)-v_t(x_t|x)$.

\begin{theorem}[Invariance of the Optimal Marginal Velocity under Frequency Weighting]
\label{thm:equivalence}
Let $\lambda_l,\lambda_h \in C([0,1])$ satisfy $\lambda_l(t), \lambda_h(t) > 0$ for all $t\in[0,1]$, and let $\mathbf{M}(t)$ be defined by Eq.~\eqref{eq:weight_matrix}. Then the weighted objective in Eq.~\eqref{eq:mahalanobis} admits
\begin{equation}
\footnotesize
\label{eq:marginal_equivalence}
\begin{aligned}
\mathcal{L}(\theta)
=
\int_0^1
\mathbb{E}
\Big[
& \Delta v_\theta(t,x_t)^\top
\mathbf{M}(t)
\Delta v_\theta(t,x_t)
\Big]
\,dt
+
C,
\end{aligned}
\end{equation}

where $v_t(x_t)=\mathbb{E}[\dot{x}_t|x_t]$ is the marginal velocity, and $C$ is a constant independent of $\theta$. Consequently, the unique minimizer of $\mathcal{L}(\theta)$, up to almost-everywhere equality, is $v_\theta^*(t,x_t)=v_t(x_t)$.
\end{theorem}

Theorem~\ref{thm:equivalence} shows that $\lambda_l(t)$ and $\lambda_h(t)$ reweight the optimization geometry without changing the population-optimal marginal velocity field induced by our path. They can thus be used to rebalance learning dynamic across frequency and across time while preserving the same target flow.

\paragraph{Final training objective.}
Our primary objective is the frequency-aligned flow matching loss in Eq.~\eqref{eq:weighted_loss}. We further add a widely-used REPA loss~\cite{yure2024presentation} on intermediate features to representation alignment. In latent diffusion~\cite{rombach2022high} and pixel diffusion~\cite{lu2026pmf,ma2026pixelgen}, perceptual loss is widely used to improve VAE image reconstruction by supervising the decoded image $\hat x$. As we adopt $x$-prediction parameterization, the LPIPS perceptual loss~\cite{zhang2018unreasonable} is a natural auxiliary objective to encourage local pattern recovery. The final objective is $\mathcal{L}_{\mathrm{total}}=\mathcal{L}+\mathcal{L}_{\mathrm{REPA}}+\mathcal{L}_{\mathrm{LPIPS}}$.

\section{Experiments}
\texttt{FREPix} is evaluated through extensive class-to-image generation experiments on ImageNet at 256 and 512 resolutions. We report FID (gFID)~\cite{heusel2017gans}, Inception Score (IS)~\cite{salimans2016improved}, Precision (Pre.) and Recall (Rec.) on 50K samples. For the frequency decomposition, we use a single-level orthonormal Haar DWT~\cite{lepik2014haar}. More details are in Appendix~D.

\subsection{Baseline Comparison}

We further conduct a baseline comparison on ImageNet 256 under the same training budget. All models are trained for 200K steps and evaluated using 50 Euler steps without CFG. The compared methods include a joint pixel-space model, JiT, and several architectures that introduce coarse-to-fine or frequency-related inductive biases implicitly, including PixDDT~\cite{wang2026ddt}, PixelFlow, PixNerd, and DeCo.

Table~\ref{tab:baseline} shows that \texttt{FREPix} achieves the best FID, IS, and precision among all compared pixel-space baselines. In particular, \texttt{FREPix} obtains 13.85 FID, 105.6 IS, and 0.67 precision, substantially outperforming the joint JiT baseline, which achieves 23.25 FID, 67.7 IS, and 0.55 precision. Compared with DeCo, a strong frequency-aware baseline, \texttt{FREPix} improves FID by 17.50 and increases IS from 48.4 to 105.6. Similar advantages are observed over PixDDT, PixNerd, and PixelFlow. These results suggest that, under the same limited training budget, explicitly modeling with frequency-heterogeneity is more effective than relying on joint prediction or implicit architectural specialization. While \texttt{FREPix} has lower recall than some baselines, its stronger FID, IS, and precision indicate a clear improvement in sample fidelity and semantic quality.

\begin{table}[t]
\centering
\small
\setlength{\tabcolsep}{2mm}
\renewcommand{\arraystretch}{0.8}
\begin{tabular}{lccccc}
\toprule
Method & Params & FID$\downarrow$ & IS$\uparrow$ & Pre.$\uparrow$ & Rec.$\uparrow$\\
\midrule
\textcolor{gray}{DiT-L/2} & \textcolor{gray}{458M+86M} & \textcolor{gray}{41.93} &\textcolor{gray}{36.5} &\textcolor{gray}{0.52} &\textcolor{gray}{0.59}\\
\textcolor{gray}{REPA-L/2} & \textcolor{gray}{458M+86M} &\textcolor{gray}{16.14} &\textcolor{gray}{87.3} &\textcolor{gray}{0.65} &\textcolor{gray}{0.63} \\
\midrule
PixDDT & 434M & 46.37 & 36.2 &0.45 & 0.63 \\
PixNerd & 458M & 37.49 & 43.0 & 0.46 & 0.62 \\
PixelFlow & 459M & 54.33 & 24.7 & 0.43 & 0.58 \\
JiT-L & 459M & 23.25 & 67.7 & 0.55 & 0.65 \\
DeCo & 426M & 31.35 & 48.4 & 0.51 & 0.65 \\
\midrule
\rowcolor[HTML]{F2F2F2} 
\texttt{FREPix} & 420M & 13.85 & 105.6 & 0.67 & 0.54 \\
\bottomrule
\end{tabular}
\caption{Baseline comparison results on ImageNet 256 after \textbf{200K} training steps. Inference uses 50 Euler steps without CFG and the REPA loss is applied to all models except DiT-L/2 and PixelFlow. \textcolor{gray}{Text in gray}: latent-space methods.}
\label{tab:baseline}
\end{table}

\subsection{Class-to-image Generation}

\begin{table*}[!t]
\centering
\small
\setlength{\tabcolsep}{1mm}
\begin{tabular}{cl|cccccccc}
\toprule
\multicolumn{2}{c|}{\textbf{Method}} & \textbf{Params} & \textbf{GFLOPs} & \textbf{Epochs} & \textbf{NFE} & \textbf{FID$\downarrow$} & \textbf{IS$\uparrow$} & \textbf{Pre.$\uparrow$} & \textbf{Rec.$\uparrow$} \\
\midrule
\multirow{14}{*}{\rotatebox{90}{256$\times$256}} 
& \textcolor{gray}{DiT-XL/2}~\cite{peebles2023scalable} & \textcolor{gray}{675M + 86M} & \textcolor{gray}{238} & \textcolor{gray}{1400} & \textcolor{gray}{250$\times$2}  & \textcolor{gray}{2.27} & \textcolor{gray}{278.2} & \textcolor{gray}{0.83} & \textcolor{gray}{0.57} \\
& \textcolor{gray}{REPA-XL/2}~\cite{yure2024presentation} & \textcolor{gray}{675M + 86M} & \textcolor{gray}{238} & \textcolor{gray}{800} & \textcolor{gray}{250$\times$2}  & \textcolor{gray}{1.42} & \textcolor{gray}{305.7} & \textcolor{gray}{0.80} & \textcolor{gray}{0.64} \\ 
\cmidrule{2-10}
& ADM~\cite{dhariwal2021diffusion} & 554M & 2240 & 400 & 250 & 4.59  & 186.7 & 0.82 & 0.52 \\
& JetFormer~\cite{tschannen2024jetformer} & 2.8B & - & - & - & 6.64 & - & 0.69 & 0.56 \\
& FractalMAR-H~\cite{lifractal} & 848M & - & 600 & - & 6.15 & 348.9 & 0.81 & 0.46 \\
& JiT-G/16~\cite{li2025back} & 2B & 766 & 600 & 100$\times$2 & 1.82  & 292.6 & - & - \\
& PixelFlow-XL/4~\cite{chen2025pixelflow} & 677M & 5818 & 320 & 120$\times$2 & 1.98  & 282.1 & 0.81 & 0.60 \\
& PixelDiT-XL~\cite{yu2025pixeldit} &  797M & 311 & 320 & 100$\times$2 & 1.61  &  292.7 & 0.78 & 0.64 \\ 
& DeCo-XL/16~\cite{ma2025deco} & 682M & 237 & 320 & 100$\times$2 & 1.90  & 303.0 & 0.80 & 0.61 \\
& PixNerd-XL/16~\cite{wang2025pixnerd} & 700M & 268 & 320 & 100$\times$2 & 2.15  & 297.0 & 0.79 & 0.59 \\
\cmidrule{2-10}
\rowcolor[HTML]{F2F2F2} &  \texttt{FREPix-XL} & 674M & 230 & 80 & 100$\times$2  & 2.29 & 294.9 & 0.79 & 0.60 \\ 
\rowcolor[HTML]{F2F2F2} & \texttt{FREPix-XL} & 674M & 230 & 320 & 100$\times$2  & 1.91 & 295.6 & 0.79 & 0.62 \\ 
\midrule
\multirow{10}{*}{\rotatebox{90}{512$\times$512}} 
& \textcolor{gray}{DiT-XL/2}~\cite{peebles2023scalable} & \textcolor{gray}{675M + 86M} & \textcolor{gray}{1050} & \textcolor{gray}{600} & \textcolor{gray}{250$\times$2}  & \textcolor{gray}{3.04} & \textcolor{gray}{240.8} & \textcolor{gray}{0.84} & \textcolor{gray}{0.54} \\
& \textcolor{gray}{SiT-XL/2}~\cite{ma2024sit} & \textcolor{gray}{675M + 86M} & \textcolor{gray}{1050}& \textcolor{gray}{600} & \textcolor{gray}{250$\times$2} & \textcolor{gray}{2.62} & \textcolor{gray}{252.2} & \textcolor{gray}{0.84} & \textcolor{gray}{0.57} \\ \cmidrule{2-10}
& ADM-G~\cite{dhariwal2021diffusion} & 554M & 3966 & 400 & 250 & 7.72 & 172.7 & 0.87 & 0.53 \\
& RIN~\cite{jabri2023scalable} & 320M & 830 & - & 250 & 3.95 & 210.0 & - & - \\
& VDM++~\cite{kingma2023understanding} & 2B & 1110 & 800 & 250$\times$2 & 2.65 & 278.1 & - & - \\
& DeCo-XL/16~\cite{ma2025deco} & 682M & 945 & 340 & 100$\times$2 & 2.22  & 290.0 & 0.80 & 0.60 \\ 
& PixelDiT-XL~\cite{yu2025pixeldit} & 797M & 1352 & 360 & 100$\times$2 &  2.21  &  271.1 & 0.78 & 0.65 \\ 
& JiT-H/32~\cite{li2025back} & 956M & - & 600 & 100$\times$2 &  1.94  & 309.1 & - & - \\ 
& PixNerd-XL/16~\cite{wang2025pixnerd} & 700M & 1042 & 340 & 100$\times$2 & 2.84  & 245.6 & 0.80 & 0.59 \\
\cmidrule{2-10}
\rowcolor[HTML]{F2F2F2} & \texttt{FREPix-XL} & 674M & 922 & 330 & 100$\times$2  & 2.38  & 334.7 & 0.80 & 0.59 \\ 
\bottomrule
\end{tabular}
\caption{Class-to-image generation on ImageNet 256×256 and 512×512 with CFG. \textcolor{gray}{Text in gray}: latent-space methods. }
\label{tab:main}
\end{table*}

\begin{figure}[!t]
  \centering
  \includegraphics[width=0.4\textwidth]{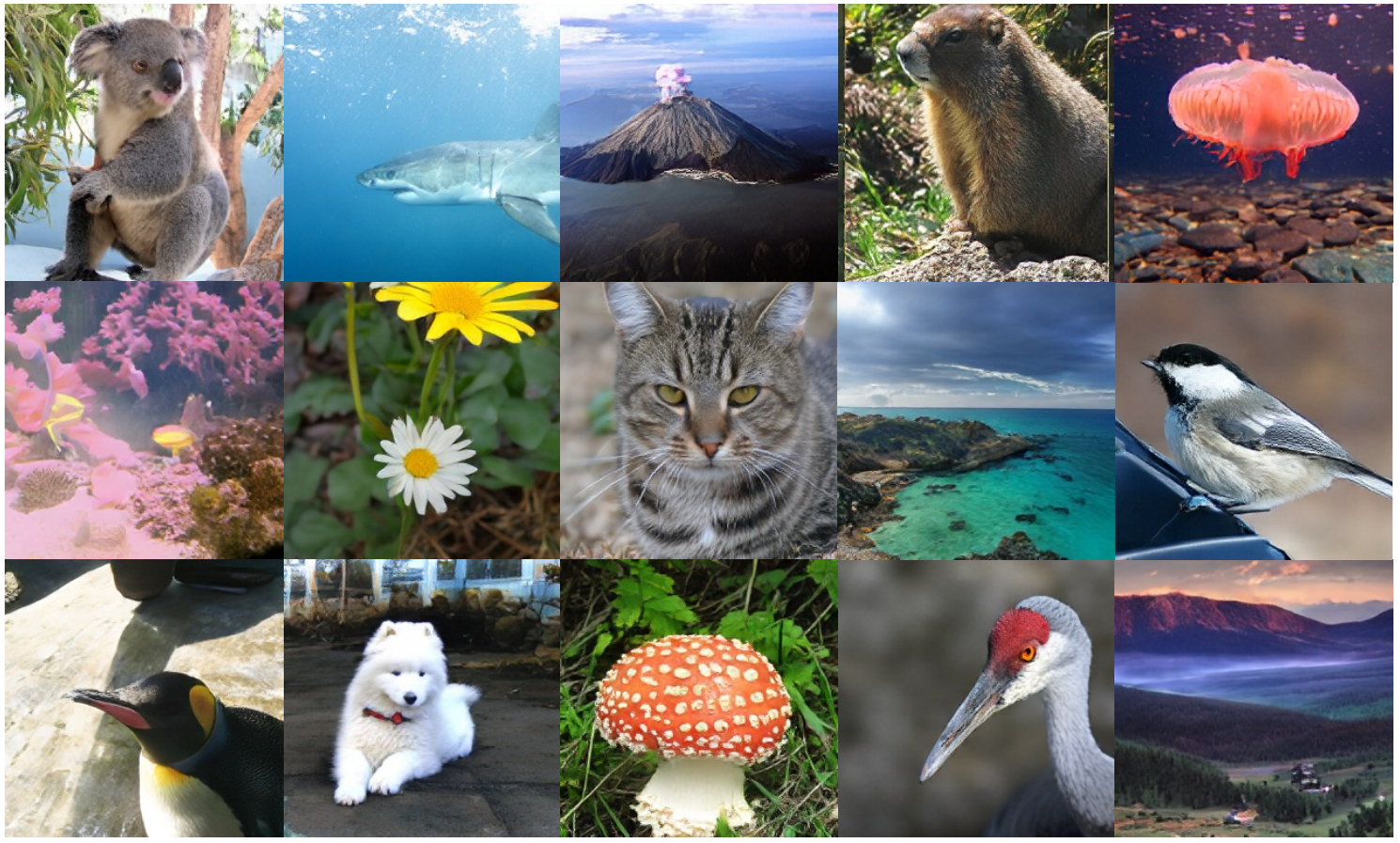}
  \caption{Qualitative results of \texttt{FREPix-XL}.}
  \label{fig:2}
\end{figure}

\begin{table}[!t]
\centering
\small
\renewcommand{\arraystretch}{0.8}
\setlength{\tabcolsep}{1mm}
\begin{tabular}{c|cccccc}
\toprule
\textbf{Method} & \textbf{Epochs} & \textbf{NFE}& \textbf{FID$\downarrow$} & \textbf{IS$\uparrow$} & \textbf{Pre.$\uparrow$} & \textbf{Rec.$\uparrow$}\\
\midrule
DeCo-XL/16 & 80 & 100 & 2.57 & - & - & - \\
PixNerd-XL/16 & 80 & 100 & 2.68 & 257.3 & 0.78 & 0.56 \\
PixelDiT & 80 & 100 & 2.36 & 282.3 & 0.80 & 0.57 \\
PixelGen & 80 & 100 & 3.64 & 317.0 & 0.79 & 0.55 \\
\rowcolor[HTML]{F2F2F2} \texttt{FREPix-XL} & 80 & 100 & 2.29 & 294.9 & 0.79 & 0.60 \\ 
\bottomrule
\end{tabular}
\caption{Comparison results under early training steps. For all models, we use the same settings.}
\label{tab:computation}
\end{table}

\begin{figure}[!t]
  \centering
  \includegraphics[width=0.42\textwidth]{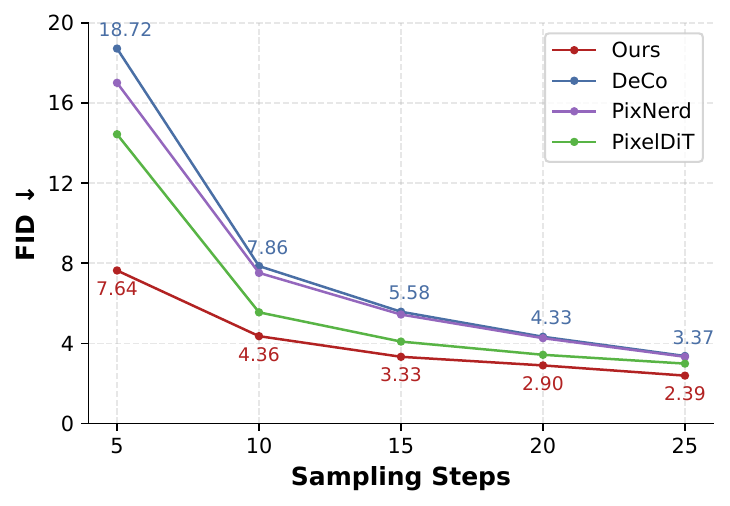}
  \caption{FID comparison in the low-NFE regime. We use the Euler sampler with CFG=3.0 from 5 to 25 sampling steps.}
  \label{fig:fid_step}
\end{figure}

\begin{figure}[!t]
  \centering
  \includegraphics[width=0.42\textwidth]{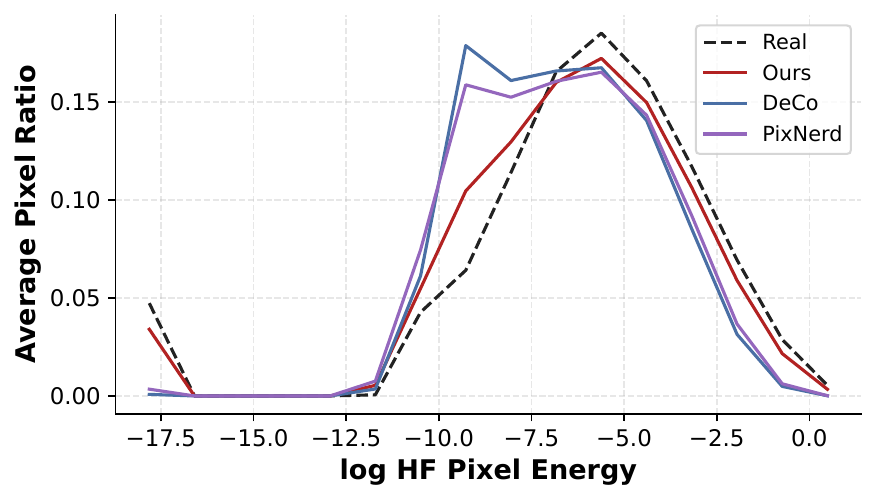}
  \caption{Visualization of the log-scale high-frequency pixel energy distributions for real images and generated samples.
  }
  \label{fig:freq_dist}
\end{figure}

The main experiments use \texttt{FREPix-XL} with 674M parameters. We sample with a 100-step Euler solver and CFG, train for 320 epochs on ImageNet 256, and further finetune for 10 epochs on ImageNet 512. Details are provided in Appendix~D.

Table~\ref{tab:main} reports the main class-conditional generation results. At $256\times256$, \texttt{FREPix-XL} achieves competitive performance among recent pixel-space models, reaching strong FID, IS, precision, and recall with a relatively lightweight model and low computational cost. Notably, as shown in Table~\ref{tab:computation}, \texttt{FREPix-XL} already performs well after only 80 training epochs, suggesting favorable optimization efficiency from the proposed frequency-factorized design. After training with 320 epochs, it further improves and remains competitive with strong pixel-space baselines, outperforming PixNerd and PixelFlow in FID while staying close to DeCo.

At $512\times512$, \texttt{FREPix-XL} continues to show strong overall performance. Although its FID is not the best among all compared methods, it achieves the highest IS while maintaining competitive precision and recall. These results suggest that explicit frequency heterogeneity remains effective when scaling to higher resolution.

The advantage of \texttt{FREPix} becomes more pronounced under reduced computation. As shown in Fig.~\ref{fig:fid_step}, \texttt{FREPix-XL} consistently achieves the lowest FID. This indicates that the frequency-heterogeneous design is particularly effective when limited computational resources. Together with its lower GFLOPs than most methods, these results suggest that \texttt{FREPix} provides a favorable trade-off between generation quality, sampling efficiency, and computational cost.

Furthermore, we visualize the log-scale pixel-energy distributions of the high-frequency components from ImageNet and generated samples in Fig.~\ref{fig:freq_dist}. Compared with baselines, \texttt{FREPix} is better aligned with the real distribution in the high-frequency subspaces, suggesting that \texttt{FREPix} more effectively captures the statistics of natural images.

\subsection{Ablation Study}
Ablation studies are conducted using \texttt{FREPix-L} at 256×256 resolution. For sampling, we take the Euler solver with 50 steps as the default choice \textbf{without} classifier-free guidance. The model is trained with 40 epochs (200K steps). More experimental details and results are provided in Appendix~D.

\paragraph{Heterogeneous interpolation path.}\label{sec:path_ablation}
We instantiate the heterogeneous interpolation path with the low-frequency path slightly ahead of the high-frequency path, i.e., $g_l(t)>g_h(t)$ for $t\in(0,1)$. Since larger $g(t)$ places the corresponding sub-state closer to its clean endpoint, this ordering exposes the model to cleaner structural information earlier, while leaving high-frequency details to be recovered later. The resulting trajectory is therefore explicitly coarse-to-fine: global structure approaches the data manifold before fine detail is fully formed. Concretely, we use smoothed power schedules
\begin{equation}
\footnotesize
g_l(t)=\frac{(t+\varepsilon)^{\gamma_l}-\varepsilon^{\gamma_l}}{(1+\varepsilon)^{\gamma_l}-\varepsilon^{\gamma_l}},
\quad
g_h(t)=\frac{(t+\varepsilon)^{\gamma_h}-\varepsilon^{\gamma_h}}{(1+\varepsilon)^{\gamma_h}-\varepsilon^{\gamma_h}},
\end{equation}
where $\gamma_l<\gamma_h$ and the offset $\varepsilon=0.01$ is a small constant to regularize the derivatives near $t=0$, while preserving the desired ordering between the two schedules.

\begin{table}[t]
    \centering
    \small
    \renewcommand{\arraystretch}{0.8}
    \setlength{\tabcolsep}{3.5mm}
    \begin{tabular}{cc|cccc}
    \toprule
    $\gamma_l$ & $\gamma_h$  & FID$\downarrow$ & IS$\uparrow$ & Pre.$\uparrow$ & Rec.$\uparrow$ \\ 
    \midrule
    0.9 & 1.1 & 14.12 & 105.0 & 0.66 & 0.54\\
    \rowcolor[HTML]{F2F2F2} 0.95 & 1.05 & 13.85 & 105.6 & 0.67 & 0.54 \\
    1.0 & 1.0 & 14.74 & 106.2 & 0.65 & 0.54 \\
    1.05 & 0.95 & 14.84 & 106.5 & 0.66 & 0.52\\ 
    1.1 & 0.9 & 15.72 & 105.1 & 0.64 & 0.52 \\ 
    \bottomrule
    \end{tabular}
    \caption{Ablation results on power exponents $\gamma_l$ and $\gamma_h$.}
    \label{tab:gamma}
\end{table}

Table~\ref{tab:gamma} shows that placing the low-frequency path ahead is important. The heterogeneous path outperforms both the homogeneous schedule and the reversed ordering, with $(\gamma_l,\gamma_h)=(0.95,1.05)$ giving the best overall FID--IS trade-off. In contrast, placing the high-frequency path ahead degrades performance. These results support our design principle that the interpolation path should reflect the asymmetric recovery process of natural images: structure should be established before high-frequency details are refined. Notably, we leave sophisticated heterogeneous path design to future work.

\paragraph{Frequency-aware reweighting.}\label{sec:loss_ablation}
Motivated by the asymmetric recovery difficulty across frequency bands, we use a time-dependent cosine schedule for $\lambda_l(t)$ and $\lambda_h(t)$. Early in generation, low-frequency structure is easier and more reliable to recover, while high-frequency details become meaningful only when the sample approaches the data manifold. We therefore emphasize low-frequency errors at early times and high-frequency errors at later times:
\begin{equation}
\footnotesize
\lambda_l(t)=1-\omega\cos\!\bigl(\pi(1-t)\bigr),
\quad
\lambda_h(t)=1+\omega\cos\!\bigl(\pi(1-t)\bigr),
\end{equation}
where $\omega$ controls the reweighting strength. Larger $\omega$ yields a stronger asymmetry across frequency.

\begin{table}[t]
    \centering
    \small
    \renewcommand{\arraystretch}{0.8}
    \setlength{\tabcolsep}{5mm}
    \begin{tabular}{c|cccc}
    \toprule
    $\omega$ & FID$\downarrow$ & IS$\uparrow$ & Pre.$\uparrow$ & Rec.$\uparrow$\\
    \midrule
    0 & 15.06 & 102.0 & 0.65 & 0.54 \\
    0.3 & 14.74 & 104.9 & 0.66 & 0.54 \\
    0.5 & 14.23 & 105.3 & 0.66 & 0.54 \\
    \rowcolor[HTML]{F2F2F2} 0.7 & 13.85 & 105.6 & 0.67 & 0.54 \\
    -0.7 & 15.49 & 99.7 & 0.65 & 0.54 \\ \bottomrule
    \end{tabular}
    \caption{Ablation results on strength $\omega$.}
    \label{tab:omega}
\end{table}

Table~\ref{tab:omega} shows that both the strength and direction are important. The proposed direction consistently outperforms its reversed counterpart, and $\omega=0.7$ gives the best overall performance. These results suggest that effective supervision should reflect the time-varying recovery difficulty of low- and high-frequency components, rather than weighting them uniformly throughout the trajectory.

\paragraph{REPA and LPIPS loss.}\label{sec:repa_ablation}

We further ablate the effects of REPA and LPIPS losses in Table~\ref{tab:loss_repa}. Both losses consistently improve pixel-space training, and their combination yields the best performance. More importantly, \texttt{FREPix} remains superior to JiT and DeCo with or without REPA, indicating that the gain does not solely come from auxiliary supervision. These results suggest that explicit frequency-factorized modeling provides an independent benefit while remaining compatible with REPA and LPIPS.

\begin{table}[t]
\centering
\small
\setlength{\tabcolsep}{1mm}
\renewcommand{\arraystretch}{0.8}
\begin{tabular}{llc|cccc}
\toprule
Method & Setting & Params & FID$\downarrow$ & IS$\uparrow$ & Pre.$\uparrow$ & Rec.$\uparrow$\\
\midrule
\multicolumn{7}{c}{\textit{Loss ablation on \texttt{FREPix-L}}} \\
\midrule
\texttt{FREPix} & Base & 420M & 53.78 & 30.34 & 0.36 & 0.56\\
\texttt{FREPix} & +LPIPS & 420M & 25.70 & 65.25 & 0.57 & 0.53\\
\texttt{FREPix} & +REPA & 420M & 33.20 & 54.95 & 0.47 & 0.60\\
\rowcolor[HTML]{F2F2F2}
\texttt{FREPix} & +REPA+LPIPS & 420M & 13.85 & 105.6 & 0.67 & 0.54\\
\midrule
\multicolumn{7}{c}{\textit{Comparison with and without REPA}} \\
\midrule
JiT-L & w/o REPA & 459M & 34.85 & 42.45 & 0.50 & 0.61\\
JiT-L & w/ REPA & 459M & 23.25 & 67.7 & 0.55 & 0.65\\
DeCo & w/o REPA & 426M & 67.55 & 19.10 & 0.47 & 0.56\\
DeCo & w/ REPA & 426M & 31.35 & 48.4 & 0.51 & 0.65\\
\rowcolor[HTML]{F2F2F2}
\texttt{FREPix} & w/o REPA & 420M & 25.70 & 65.25 & 0.57 & 0.53\\
\rowcolor[HTML]{F2F2F2}
\texttt{FREPix} & w/ REPA & 420M & 13.85 & 105.6 & 0.67 & 0.54\\
\bottomrule
\end{tabular}
\caption{Ablation of REPA and LPIPS losses, and comparison with JiT and DeCo under different REPA settings.}
\label{tab:loss_repa}
\end{table}

\section{Conclusion}

In this paper, we propose a frequency-heterogeneous flow matching framework  for pixel-space image generation (\texttt{FREPix}). Our starting point is that natural images are inherently heterogeneous across frequencies, whereas existing pixel-space generation methods still largely formulate generation as a frequency-homogeneous process. \texttt{FREPix} makes this heterogeneity explicit throughout generation. Extensive experiments on ImageNet demonstrate that this formulation yields competitive performance among pixel-space generation models. We hope this work highlights frequency heterogeneity as a useful perspective for designing future pixel-space generative models.

\bibliography{reference}

@inproceedings{rombach2022high,
  title={High-resolution image synthesis with latent diffusion models},
  author={Rombach, Robin and Blattmann, Andreas and Lorenz, Dominik and Esser, Patrick and Ommer, Bj{\"o}rn},
  booktitle={Proceedings of the IEEE/CVF conference on computer vision and pattern recognition},
  pages={10684--10695},
  year={2022}
}

@inproceedings{peebles2023scalable,
  title={Scalable diffusion models with transformers},
  author={Peebles, William and Xie, Saining},
  booktitle={Proceedings of the IEEE/CVF international conference on computer vision},
  pages={4195--4205},
  year={2023}
}

@inproceedings{ma2024sit,
  title={Sit: Exploring flow and diffusion-based generative models with scalable interpolant transformers},
  author={Ma, Nanye and Goldstein, Mark and Albergo, Michael S and Boffi, Nicholas M and Vanden-Eijnden, Eric and Xie, Saining},
  booktitle={European Conference on Computer Vision},
  pages={23--40},
  year={2024},
  organization={Springer}
}

@article{ho2020denoising,
  title={Denoising diffusion probabilistic models},
  author={Ho, Jonathan and Jain, Ajay and Abbeel, Pieter},
  journal={Advances in neural information processing systems},
  volume={33},
  pages={6840--6851},
  year={2020}
}

@inproceedings{nichol2021improved,
  title={Improved denoising diffusion probabilistic models},
  author={Nichol, Alexander Quinn and Dhariwal, Prafulla},
  booktitle={International conference on machine learning},
  pages={8162--8171},
  year={2021},
  organization={PMLR}
}

@article{li2025back,
  title={Back to basics: Let denoising generative models denoise},
  author={Li, Tianhong and He, Kaiming},
  journal={arXiv preprint arXiv:2511.13720},
  year={2025}
}

@inproceedings{ma2025deco,
  title={Deco: Frequency-decoupled pixel diffusion for end-to-end image generation},
  author={Ma, Zehong and Wei, Longhui and Wang, Shuai and Zhang, Shiliang and Tian, Qi},
  booktitle={Proceedings of the IEEE/CVF Conference on Computer Vision and Pattern Recognition},
  pages={43600--43610},
  year={2026}
}

@inproceedings{yu2025pixeldit,
  title={Pixeldit: Pixel diffusion transformers for image generation},
  author={Yu, Yongsheng and Xiong, Wei and Nie, Weili and Sheng, Yichen and Liu, Shiqiu and Luo, Jiebo},
  booktitle={Proceedings of the IEEE/CVF Conference on Computer Vision and Pattern Recognition},
  pages={14273--14282},
  year={2026}
}

@article{wang2025pixnerd,
  title={Pixnerd: Pixel neural field diffusion},
  author={Wang, Shuai and Gao, Ziteng and Zhu, Chenhui and Huang, Weilin and Wang, Limin},
  journal={arXiv preprint arXiv:2507.23268},
  year={2025}
}

@inproceedings{lipman2022flow,
  title={Flow matching for generative modeling},
  author={Lipman, Yaron and Chen, Ricky TQ and Ben-Hamu, Heli and Nickel, Maximilian and Le, Matt},
  booktitle={The Eleventh International Conference on Learning Representations},
  year={2023}
}

@inproceedings{liu2022flow,
  title={Flow Straight and Fast: Learning to Generate and Transfer Data with Rectified Flow},
  author={Liu, Xingchao and Gong, Chengyue and others},
  booktitle={The Eleventh International Conference on Learning Representations},
  year={2023}
}

@article{albergo2025stochastic,
  title={Stochastic interpolants: A unifying framework for flows and diffusions},
  author={Albergo, Michael and Boffi, Nicholas M and Vanden-Eijnden, Eric},
  journal={Journal of Machine Learning Research},
  volume={26},
  number={209},
  pages={1--80},
  year={2025}
}

@article{ho2022cascaded,
  title={Cascaded diffusion models for high fidelity image generation},
  author={Ho, Jonathan and Saharia, Chitwan and Chan, William and Fleet, David J and Norouzi, Mohammad and Salimans, Tim},
  journal={Journal of Machine Learning Research},
  volume={23},
  number={47},
  pages={1--33},
  year={2022}
}

@inproceedings{phung2023wavelet,
  title={Wavelet diffusion models are fast and scalable image generators},
  author={Phung, Hao and Dao, Quan and Tran, Anh},
  booktitle={Proceedings of the IEEE/CVF conference on computer vision and pattern recognition},
  pages={10199--10208},
  year={2023}
}

@article{chen2025pixelflow,
  title={Pixelflow: Pixel-space generative models with flow},
  author={Chen, Shoufa and Ge, Chongjian and Zhang, Shilong and Sun, Peize and Luo, Ping},
  journal={arXiv preprint arXiv:2504.07963},
  year={2025}
}

@article{percival1995estimation,
  title={On estimation of the wavelet variance},
  author={Percival, Donald P},
  journal={Biometrika},
  volume={82},
  number={3},
  pages={619--631},
  year={1995},
  publisher={Oxford University Press}
}

@book{pollard1990empirical,
  title={Empirical processes: theory and applications},
  author={Pollard, David},
  year={1990},
  organization={Ims}
}

@book{mohri2018foundations,
  title={Foundations of machine learning},
  author={Mohri, Mehryar and Rostamizadeh, Afshin and Talwalkar, Ameet},
  year={2018},
  publisher={MIT press}
}

@article{dudley1987universal,
  title={Universal Donsker Classes and Metric Entropy},
  author={Dudley, RM},
  journal={The Annals of Probability},
  volume={15},
  number={4},
  pages={1306--1326},
  year={1987},
  publisher={Institute of Mathematical Statistics}
}

@article{dudley1967sizes,
  title={The sizes of compact subsets of Hilbert space and continuity of Gaussian processes},
  author={Dudley, Richard M},
  journal={Journal of Functional Analysis},
  volume={1},
  number={3},
  pages={290--330},
  year={1967},
  publisher={Elsevier}
}

@inproceedings{yure2024presentation,
  title={Representation Alignment for Generation: Training Diffusion Transformers Is Easier Than You Think},
  author={Yu, Sihyun and Kwak, Sangkyung and Jang, Huiwon and Jeong, Jongheon and Huang, Jonathan and Shin, Jinwoo and Xie, Saining},
  booktitle={The Thirteenth International Conference on Learning Representations},
  year={2024}
}

@inproceedings{zhang2018unreasonable,
  title={The unreasonable effectiveness of deep features as a perceptual metric},
  author={Zhang, Richard and Isola, Phillip and Efros, Alexei A and Shechtman, Eli and Wang, Oliver},
  booktitle={Proceedings of the IEEE conference on computer vision and pattern recognition},
  pages={586--595},
  year={2018}
}

@article{dhariwal2021diffusion,
  title={Diffusion models beat gans on image synthesis},
  author={Dhariwal, Prafulla and Nichol, Alexander},
  journal={Advances in neural information processing systems},
  volume={34},
  pages={8780--8794},
  year={2021}
}

@inproceedings{leng2025repa,
  title={Repa-e: Unlocking vae for end-to-end tuning of latent diffusion transformers},
  author={Leng, Xingjian and Singh, Jaskirat and Hou, Yunzhong and Xing, Zhenchang and Xie, Saining and Zheng, Liang},
  booktitle={Proceedings of the IEEE/CVF International Conference on Computer Vision},
  pages={18262--18272},
  year={2025}
}

@inproceedings{yao2025reconstruction,
  title={Reconstruction vs. generation: Taming optimization dilemma in latent diffusion models},
  author={Yao, Jingfeng and Yang, Bin and Wang, Xinggang},
  booktitle={Proceedings of the Computer Vision and Pattern Recognition Conference},
  pages={15703--15712},
  year={2025}
}

@article{shi2025latent,
  title={Latent diffusion model without variational autoencoder},
  author={Shi, Minglei and Wang, Haolin and Zheng, Wenzhao and Yuan, Ziyang and Wu, Xiaoshi and Wang, Xintao and Wan, Pengfei and Zhou, Jie and Lu, Jiwen},
  journal={arXiv preprint arXiv:2510.15301},
  year={2025}
}

@inproceedings{teng2024relay,
  title={Relay Diffusion: Unifying diffusion process across resolutions for image synthesis},
  author={Teng, Jiayan and Zheng, Wendi and Ding, Ming and Hong, Wenyi and Wangni, Jianqiao and Yang, Zhuoyi and Tang, Jie},
  booktitle={The Twelfth International Conference on Learning Representations},
  year={2024}
}

@inproceedings{rahaman2019spectral,
  title={On the spectral bias of neural networks},
  author={Rahaman, Nasim and Baratin, Aristide and Arpit, Devansh and Draxler, Felix and Lin, Min and Hamprecht, Fred and Bengio, Yoshua and Courville, Aaron},
  booktitle={International conference on machine learning},
  pages={5301--5310},
  year={2019},
  organization={PMLR}
}

@article{torralba2003statistics,
  title={Statistics of natural image categories},
  author={Torralba, Antonio and Oliva, Aude},
  journal={Network: computation in neural systems},
  volume={14},
  number={3},
  pages={391},
  year={2003},
  publisher={IOP Publishing}
}

@inproceedings{chen2019drop,
  title={Drop an octave: Reducing spatial redundancy in convolutional neural networks with octave convolution},
  author={Chen, Yunpeng and Fan, Haoqi and Xu, Bing and Yan, Zhicheng and Kalantidis, Yannis and Rohrbach, Marcus and Yan, Shuicheng and Feng, Jiashi},
  booktitle={Proceedings of the IEEE/CVF international conference on computer vision},
  pages={3435--3444},
  year={2019}
}

@article{xu2020frequency,
  title={Frequency Principle: Fourier Analysis Sheds Light on Deep Neural Networks},
  author={Xu, Zhi-Qin John},
  journal={Communications in Computational Physics},
  volume={28},
  number={5},
  pages={1746--1767},
  year={2020}
}

@article{karras2022elucidating,
  title={Elucidating the design space of diffusion-based generative models},
  author={Karras, Tero and Aittala, Miika and Aila, Timo and Laine, Samuli},
  journal={Advances in neural information processing systems},
  volume={35},
  pages={26565--26577},
  year={2022}
}

@inproceedings{tschannen2024jetformer,
  title={JetFormer: An autoregressive generative model of raw images and text},
  author={Tschannen, Michael and Pinto, Andr{\'e} Susano and Kolesnikov, Alexander},
  booktitle={The Thirteenth International Conference on Learning Representations},
  year={2024}
}

@article{lifractal,
  title={Fractal Generative Models},
  author={Li, Tianhong and Sun, Qinyi and Fan, Lijie and He, Kaiming},
  journal={Transactions on Machine Learning Research},
  year={2025}
}

@inproceedings{jabri2023scalable,
  title={Scalable Adaptive Computation for Iterative Generation},
  author={Jabri, Allan and Fleet, David J and Chen, Ting},
  booktitle={International Conference on Machine Learning},
  pages={14569--14589},
  year={2023},
  organization={PMLR}
}

@article{kingma2023understanding,
  title={Understanding diffusion objectives as the elbo with simple data augmentation},
  author={Kingma, Diederik and Gao, Ruiqi},
  journal={Advances in Neural Information Processing Systems},
  volume={36},
  pages={65484--65516},
  year={2023}
}

@article{heusel2017gans,
  title={Gans trained by a two time-scale update rule converge to a local nash equilibrium},
  author={Heusel, Martin and Ramsauer, Hubert and Unterthiner, Thomas and Nessler, Bernhard and Hochreiter, Sepp},
  journal={Advances in neural information processing systems},
  volume={30},
  year={2017}
}

@article{salimans2016improved,
  title={Improved techniques for training gans},
  author={Salimans, Tim and Goodfellow, Ian and Zaremba, Wojciech and Cheung, Vicki and Radford, Alec and Chen, Xi},
  journal={Advances in neural information processing systems},
  volume={29},
  year={2016}
}

@incollection{lepik2014haar,
  title={Haar wavelets},
  author={Lepik, {\"U}lo and Hein, Helle},
  booktitle={Haar wavelets: with applications},
  pages={7--20},
  year={2014},
  publisher={Springer}
}

@inproceedings{hoogeboom2025simpler,
  title={Simpler Diffusion: 1.5 FID on ImageNet512 with pixel-space diffusion},
  author={Hoogeboom, Emiel and Mensink, Thomas and Heek, Jonathan and Lamerigts, Kay and Gao, Ruiqi and Salimans, Tim},
  booktitle={Proceedings of the Computer Vision and Pattern Recognition Conference},
  pages={18062--18071},
  year={2025}
}

@article{kynkaanniemi2024applying,
  title={Applying guidance in a limited interval improves sample and distribution quality in diffusion models},
  author={Kynk{\"a}{\"a}nniemi, Tuomas and Aittala, Miika and Karras, Tero and Laine, Samuli and Aila, Timo and Lehtinen, Jaakko},
  journal={Advances in Neural Information Processing Systems},
  volume={37},
  pages={122458--122483},
  year={2024}
}

@article{falck2025fourier,
  title={A fourier space perspective on diffusion models},
  author={Falck, Fabian and Pandeva, Teodora and Zahirnia, Kiarash and Lawrence, Rachel and Turner, Richard and Meeds, Edward and Zazo, Javier and Karmalkar, Sushrut},
  journal={arXiv preprint arXiv:2505.11278},
  year={2025}
}

@inproceedings{wang2026ddt,
  title={Ddt: Decoupled diffusion transformer},
  author={Wang, Shuai and Tian, Zhi and Huang, Weilin and Wang, Limin},
  booktitle={Proceedings of the IEEE/CVF Conference on Computer Vision and Pattern Recognition},
  pages={40633--40642},
  year={2026}
}

@article{lu2026pmf,
  title={One-step Latent-free Image Generation with Pixel Mean Flows},
  author={Lu, Yiyang and Lu, Susie and Sun, Qiao and Zhao, Hanhong and Jiang, Zhicheng and Wang, Xianbang and Li, Tianhong and Geng, Zhengyang and He, Kaiming},
  journal={arXiv preprint arXiv:2601.22158},
  year={2026}
}

@article{ma2026pixelgen,
  title={PixelGen: Pixel Diffusion Beats Latent Diffusion with Perceptual Loss},
  author={Ma, Zehong and Xu, Ruihan and Zhang, Shiliang},
  journal={arXiv preprint arXiv:2602.02493},
  year={2026}
}
\clearpage

\appendix

\section{Limitations}\label{app:limitations}

\texttt{FREPix} has several limitations. First, we only explore a limited family of heterogeneous interpolation schedules. While our results show that asymmetric low-/high-frequency transport is beneficial, the best schedule design remains underexplored, and richer or adaptive parameterizations may further improve performance. Second, \texttt{FREPix} is instantiated with a fixed orthonormal wavelet decomposition. This choice provides an exact and simple frequency factorization, but it is not the only way to expose heterogeneous image structure. More flexible multi-resolution or learned decompositions may better match the statistics of natural images and further improve the framework. Finally, our theoretical results are derived under simplified assumptions and are mainly intended to support the design intuition of explicit network decoupling, rather than to provide a complete characterization of modern large-scale architectures. We hope these limitations motivate future work on more flexible frequency decompositions, richer schedule designs, and broader empirical validation.

\section{Preliminaries}
Flow-based generative models~\cite{lipman2022flow,liu2022flow,albergo2025stochastic} define sampling as simulating an ODE that pushes a prior distribution $\rho_0$ (typically $\mathcal N(0,I_d)$) forward to the data distribution $\rho_1$. During training, a noisy sample $x_t$ is typically constructed using a simple linear interpolation path:
\begin{equation} \footnotesize
    x_t = t\,x + (1-t)\,\epsilon,\qquad t\in[0,1],
    \label{eq:linear}
\end{equation}
where $x \sim \rho_1$ and $\epsilon \sim \rho_0$ denote the clean data and noise. Here, $t \in [0, 1]$ dictates the generative trajectory from the initial noise state ($t=0$) to the clean data ($t=1$). This interpolation path induces the conditional velocity field $v_t(x_t\mid x)=x-\epsilon$. Conditional Flow Matching (CFM)~\cite{lipman2022flow} learns a time-dependent network $v_\theta$ via $L^2$-regression against this target:
\begin{equation} \footnotesize
    \mathcal L_{\rm CFM}(\theta)=\mathbb E_{t,\;x\sim\rho_1,\;\epsilon\sim\rho_0}\Big[\bigl\|v_\theta(t,x_t)-v_t(x_t\mid x)\bigr\|^2\Big].
    \label{eq:cfm_loss}
\end{equation}
Once trained, new samples are obtained by integrating the ODE
\begin{equation} \footnotesize
    \frac{d}{dt}x_t = v_\theta(t,x_t),\qquad t\in[0,1],
\end{equation}
starting from $t=0$ and ending at $t = 1$. In practice, this ODE can be approximately solved using numerical solvers (e.g., Euler- and Heun-based solvers~\cite{karras2022elucidating}).

\section{Proofs}\label{app:proof}

\subsection{Proofs of Proposition~1}\label{app:proof1}
\begin{proof}
Recall the heterogeneous interpolation path
\begin{equation}
\footnotesize
\begin{aligned}
    G(t)
    &:=\mathcal W^{-1} \begin{pmatrix} g_l(t) I_{d_l} & 0\\ 0 & g_h(t) I_{d_h} \end{pmatrix} \mathcal W, 
    \\
    \qquad
    x_t
    &= G(t)x + (I-G(t))\epsilon.
\end{aligned}
\end{equation}

By the orthonormality of $\mathcal{W}$ and the regularity of $g_l,g_h\in C^1([0,1])$, the matrix-valued map $t\mapsto G(t)$ is continuously differentiable. Define
\begin{equation} \footnotesize
    L_g := \max\Big\{\sup_{t\in[0,1]} |\dot g_l(t)|,\ \sup_{t\in[0,1]} |\dot g_h(t)|\Big\} < \infty .
\end{equation}

Since orthonormal changes of coordinates preserve operator norms,
\begin{equation} \footnotesize
\begin{aligned}
    \|\dot G(t)\|_{\mathrm{op}} &= \max\bigl\{ |\dot g_l(t)|, |\dot g_h(t)| \bigr\} \le L_g,
    \\
    \|G(t)\|_{\mathrm{op}} &= \max\bigl\{ |g_l(t)|, |g_h(t)| \bigr\} \le 1.
\end{aligned}
\end{equation}

\paragraph{Step 1: Smoothness.}
For each realization of $(x,\epsilon)$, the path $t\mapsto x_t$ is $C^1$, with $ \dot x_t = \dot G(t)(x-\epsilon)$, yielding
\begin{equation} \footnotesize
    \|\dot x_t\| \le \|\dot G(t)\|_{\mathrm{op}} \, \|x-\epsilon\| \le L_g(\|x\|+\|\epsilon\|),
\end{equation}
which establishes the claimed bound. Furthermore, given that $x$ has finite second moment and $\epsilon\sim \mathcal N(0,I_d)$,
\begin{equation} \footnotesize 
    \int_0^1 \mathbb{E}\|\dot x_t\|^2\,dt \le 2L_g^2\bigl(\mathbb{E}\|x\|^2 + \mathbb{E}\|\epsilon\|^2\bigr) <\infty.
\end{equation}

\paragraph{Step 2: Density and continuity equation.}
Fix $t\in[0,1)$. The strict monotonicity of $g_l$ and $g_h$ together with the boundary conditions $g_l(1)=g_h(1)=1$ yields $0\le g_l(t)<1$ and $0\le g_h(t)<1$. Therefore,
\begin{equation} \footnotesize 
    I-G(t)=\mathcal W^{-1}\begin{pmatrix}(1-g_l(t)) I_{d_l} & 0\\0 & (1-g_h(t)) I_{d_h}\end{pmatrix}\mathcal W
\end{equation}
is invertible, and the conditional law of $x_t$ given $x$ is Gaussian:
\begin{equation} \footnotesize 
    x_t\mid x \sim \mathcal N\!\bigl(G(t)x,\ \Sigma_t\bigr), \quad \Sigma_t := (I-G(t))(I-G(t))^\top .
\end{equation}

By the positive definiteness of $\Sigma_t$ for every $t<1$, the law of $x_t$ has the density
\begin{equation} \footnotesize 
    p_t(z) = \int_{\mathbb R^d}\phi_{\Sigma_t}\!\bigl(z-G(t)x\bigr)\,\rho_1(dx),
\end{equation}
where $\phi_{\Sigma_t}$ denotes the Gaussian density with covariance $\Sigma_t$. In particular, $p_t$ is well-defined for every $t\in[0,1)$.

As $t\mapsto x_t$ is almost surely $C^1$,
the chain rule gives
\begin{equation} \footnotesize 
    \frac{d}{dt}\varphi(x_t) = \nabla \varphi(x_t)\cdot \dot x_t .
\end{equation} 

Moreover,
\begin{equation} \footnotesize 
    \biggl|\frac{d}{dt}\varphi(x_t)\biggr| \le \|\nabla \varphi\|_\infty\,\|\dot x_t\|,
\end{equation} 
and the right-hand side is integrable by Step 1. Hence, by dominated convergence,
\begin{equation} \footnotesize 
    \frac{d}{dt}\mathbb E[\varphi(x_t)] = \mathbb E\!\left[\nabla\varphi(x_t)\cdot \dot x_t\right].
\end{equation} 

Define the marginal velocity field $v_t(x_t) := \mathbb E[\dot x_t\mid x_t]$. Using the tower property,
\begin{equation} \footnotesize 
\begin{aligned}
    \mathbb E\!\left[\nabla\varphi(x_t)\cdot \dot x_t\right] &= \mathbb E\!\left[\nabla\varphi(x_t)\cdot v_t(x_t)\right] \\&= \int_{\mathbb R^d} \nabla\varphi(z)\cdot v_t(z)\,p_t(z)\,dz.
\end{aligned}
\end{equation} 

Direct computation yields
\begin{equation} \footnotesize 
    \mathbb E[\varphi(x_t)] = \int_{\mathbb R^d} \varphi(z)\,p_t(z)\,dz.
\end{equation} 

Therefore,
\begin{equation} \footnotesize
\begin{aligned}
    \frac{d}{dt}\int_{\mathbb R^d}\varphi(z)\,p_t(z)\,dz &= \int_{\mathbb R^d}\nabla\varphi(z)\cdot v_t(z)\,p_t(z)\,dz \\&= -\int_{\mathbb R^d}\varphi(z)\,\nabla\!\cdot\!\bigl(v_t(z)p_t(z)\bigr)\,dz.
\end{aligned}
\end{equation} 

As this identity holds for every $\varphi\in C_c^\infty(\mathbb R^d)$, it follows that
\begin{equation} \footnotesize
    \partial_t p_t + \nabla\cdot(v_t p_t)=0
\end{equation} 
in the sense of distributions on $\mathbb R^d$.

\paragraph{Step 3: Learnability.}
Let
\begin{equation} \footnotesize 
    \mathcal B = \Bigl\{ b:\ [0,1]\times\mathbb R^d\to\mathbb R^d:\; \int_0^1 \mathbb E\|b(t,x_t)\|^2\,dt < \infty \Bigr\}.
\end{equation} 

By Jensen's inequality for conditional expectations and Step 1,
\begin{equation}
\footnotesize
\begin{aligned}
\int_0^1
\mathbb{E}\|v_t(x_t)\|^2\,dt
&=
\int_0^1
\mathbb{E}
\bigl\|
\mathbb{E}[\dot{x}_t\mid x_t]
\bigr\|^2
\,dt
\\
&\le
\int_0^1
\mathbb{E}\|\dot{x}_t\|^2\,dt
<
\infty .
\end{aligned}
\end{equation}

which implies $v\in\mathcal B$.

Now fix any $b\in\mathcal B$. Expansion of the squared norm yields
\begin{equation}
\footnotesize
\begin{aligned}
\|b(t,x_t)-\dot{x}_t\|^2
&=
\|b(t,x_t)-v_t(x_t)\|^2
+
\|v_t(x_t)-\dot{x}_t\|^2
\\
&\quad
+
2\left\langle
b(t,x_t)-v_t(x_t),
\,v_t(x_t)-\dot{x}_t
\right\rangle .
\end{aligned}
\end{equation}

Taking expectations, the cross term vanishes:
\begin{equation} \footnotesize 
\begin{aligned}
&\mathbb E\!\left[\langle b(t,x_t)-v_t(x_t),\, v_t(x_t)-\dot x_t\rangle\right] \\
&\quad=
\mathbb E\!\left[
\mathbb E\!\left[
\langle b(t,x_t)-v_t(x_t),\, v_t(x_t)-\dot x_t\rangle
\,\middle|\, x_t
\right]\right] \\
&\quad=
\mathbb E\!\left[
\left\langle b(t,x_t)-v_t(x_t),\,
\mathbb E[v_t(x_t)-\dot x_t\mid x_t]\right\rangle
\right]
=0,
\end{aligned}
\end{equation} 
since $b(t,x_t)-v_t(x_t)$ is $\sigma(x_t)$-measurable and
$\mathbb E[v_t(x_t)-\dot x_t\mid x_t] = v_t(x_t)-\mathbb E[\dot x_t\mid x_t] =0.$

Integrating over $t$ yields the orthogonal decomposition
\begin{equation} \footnotesize 
    \mathcal L(b)=\mathcal L(v)+\int_0^1 \mathbb E\|b(t,x_t)-v_t(x_t)\|^2\,dt .
\end{equation} 

Hence $\mathcal L(b)\ge \mathcal L(v)$ for every $b\in\mathcal B$, with equality if and only if
\begin{equation} \footnotesize 
    b(t,x_t)=v_t(x_t) \; \text{for Lebesgue-a.e. } t\in[0,1] \text{ and } p_t\text{-a.e. } x_t\in\mathbb R^d.
\end{equation} 

Therefore, the population regression objective is uniquely minimized, up to almost-everywhere equality, by
\begin{equation} \footnotesize 
b^*(t,x_t)=v_t(x_t)=\mathbb E[\dot x_t\mid x_t].
\end{equation} 

This proves the proposition.
\end{proof}

\subsection{Proofs of Proposition~2}\label{app:proof2}

Let $\mathcal W:\mathbb R^d \to \mathbb R^{d_L}\times \mathbb R^{d_H}$ be an orthonormal discrete wavelet transform, where $d=d_L+d_H$. For any sample $(x_t,x)$, write its wavelet decomposition as
\begin{equation} \footnotesize 
    \mathcal W(x_t)=(l_t,h_t),\qquad \mathcal W(x)=(l,h).
\end{equation} 
Given $N$ i.i.d.\ samples $S=\{(l_t^{(i)},h_t^{(i)},l^{(i)},h^{(i)})\}_{i=1}^N$, we compare direct modeling with explicit decoupling under a simplified analysis, and then extend the result to the practical architecture.

\begin{definition}[Function classes]
\label{def:function}
The direct modeling class takes the full noisy state $(l_t,h_t)\in\mathbb R^d$ and jointly predicts the clean low- and high-frequency components:
\begin{equation} \footnotesize 
\begin{aligned}
    \mathcal F_{\mathrm{dir}} &:= \mathcal F_{\mathrm{dir}}^L \times \mathcal F_{\mathrm{dir}}^H, \\ \mathcal F_{\mathrm{dir}}^L &\subset \{f_L:\mathbb R^d\to\mathbb R^{d_L}\}, \\ \mathcal F_{\mathrm{dir}}^H &\subset \{f_H:\mathbb R^d\to\mathbb R^{d_H}\}.
\end{aligned}
\end{equation} 

The decoupled function class separates the prediction responsibilities: the low-frequency branch only takes $l_t$, while the high-frequency branch is analyzed under teacher forcing and takes $(l,h_t)$:
\begin{equation} \footnotesize 
\begin{aligned}
    \mathcal F_{\mathrm{dec}} &:= \mathcal F_{\mathrm{dec}}^L \times \mathcal F_{\mathrm{dec}}^H, \\ \mathcal F_{\mathrm{dec}}^L &\subset \{g_L:\mathbb R^{d_L}\to\mathbb R^{d_L}\}, \\ \mathcal F_{\mathrm{dec}}^H &\subset \{g_H:\mathbb R^{d_L}\times\mathbb R^{d_H}\to\mathbb R^{d_H}\}.
\end{aligned}
\end{equation} 

The practical function class feeds the predicted low-frequency component into the high-frequency branch:
\begin{equation} \footnotesize 
    \mathcal F_{\mathrm{real}} := \Bigl\{ (g_L,g_H): g_L\in\mathcal F_{\mathrm{dec}}^L,\; g_H\in\mathcal F_{\mathrm{dec}}^H \Bigr\},
\end{equation} 
with prediction rule $\hat l = g_L(l_t), \hat h = g_H(\hat l,h_t).$

\end{definition}

\begin{assumption}[Boundedness]
\label{ass:boundedness}
There exists $B>0$ such that almost surely $\|l\|_\infty,\ \|h\|_\infty \le B$, and all candidate predictors satisfy $\|\hat l\|_\infty,\ \|\hat h\|_\infty \le B$. Define the normalized squared losses
\begin{equation} \footnotesize 
    \ell_L(\hat l,l):=\frac{\|\hat l-l\|_2^2}{4B^2d_L}, \qquad \ell_H(\hat h,h):=\frac{\|\hat h-h\|_2^2}{4B^2d_H}.
\end{equation} 

Then $0\le \ell_L,\ell_H\le 1$.
\end{assumption}

\begin{assumption}[Low-dimensional structural manifold]
\label{ass:manifold}
The clean low-frequency component $l$ is concentrated near a low-dimensional manifold $\mathcal M_L \subset \mathbb R^{d_L}$ with intrinsic dimension $k_L<d_L$.
\end{assumption}

\begin{assumption}[Covering-number growth under a finite-dimensional proxy analysis]
\label{ass:covering}
Let the loss classes induced by the normalized squared losses be
\begin{equation} \footnotesize
\begin{aligned}
    \mathcal G_{\mathrm{dir}}^L &:= \{z\mapsto \ell_L(f_L(l_t,h_t),l): f_L\in\mathcal F_{\mathrm{dir}}^L\}, 
    \\\mathcal G_{\mathrm{dir}}^H &:= \{z\mapsto \ell_H(f_H(l_t,h_t),h): f_H\in\mathcal F_{\mathrm{dir}}^H\}, \\ \mathcal G_{\mathrm{dec}}^L &:= \{z\mapsto \ell_L(g_L(l_t),l): g_L\in\mathcal F_{\mathrm{dec}}^L\}, \\ \mathcal G_{\mathrm{dec}}^H &:= \{z\mapsto \ell_H(g_H(l,h_t),h): g_H\in\mathcal F_{\mathrm{dec}}^H\},
\end{aligned}
\end{equation}
where $z=(l_t,h_t,l,h)$.

We assume a simplified finite-dimensional proxy analysis in which each loss class $\mathcal G$ admits an effective parameterization of dimension $m_{\mathcal G}$ over a bounded parameter set, and the induced loss is uniformly Lipschitz with respect to that parameterization. Under this proxy, the metric entropy satisfies the Pollard-type growth condition~\cite{pollard1990empirical}: there exists a constant $A\ge 1$ such that, for every $\varepsilon\in(0,1]$ and every $\mathcal{G}\in\{\mathcal{G}^{L}_{\mathrm{dir}},\mathcal{G}^{H}_{\mathrm{dir}},\mathcal{G}^{L}_{\mathrm{dec}},\mathcal{G}^{H}_{\mathrm{dec}}\},$

\begin{equation} \footnotesize
\label{eq:covering_unified}
    \log\mathcal{N}(\varepsilon,\mathcal{G},L_2(P_N)) \le m_{\mathcal G}\log(A/\varepsilon).
\end{equation}

In the simplified linear proxy considered here, the effective dimensions scale with the corresponding input degrees of freedom:
\begin{equation} \footnotesize
    m_{\mathcal{G}^{L}_{\mathrm{dir}}}= m_{\mathcal{G}^{H}_{\mathrm{dir}}}=d, \quad m_{\mathcal{G}^{L}_{\mathrm{dec}}}=d_L, \quad m_{\mathcal{G}^{H}_{\mathrm{dec}}}=k_L+d_H,
\end{equation}
where the last relation reflects that the high-frequency branch is conditioned on $(l,h_t)$ and the clean low-frequency component $l$ is assumed to lie near a $k_L$-dimensional structural manifold.
\end{assumption}

\begin{assumption}[Conditional Lipschitz property]
\label{ass:lipschitz}
There exists $L_{\mathrm{cond}}>0$ such that for every $z,z'\in\mathbb R^{d_L}$ and every $h_t\in\mathbb R^{d_H}$,
\begin{equation}
\footnotesize
\|g_H(z,h_t)-g_H(z',h_t)\|_2
\le
L_{\mathrm{cond}}\|z-z'\|_2,
\quad
\forall g_H\in\mathcal F_{\mathrm{dec}}^H.
\end{equation}

This assumption quantifies the error propagation induced by replacing the clean low-frequency input $l$ with its prediction $\hat l=g_L(l_t)$ in the practical architecture.
\end{assumption}

\subsubsection{From covering numbers to Rademacher complexity}
\label{app:entropy_rademacher}

\begin{lemma}[Entropy integral bound]
\label{lem:entropy_to_rademacher}
Let $\mathcal G\subset [0,1]^{\mathcal X}$ be a function class satisfying
\begin{equation} \footnotesize 
    \log \mathcal N(\varepsilon,\mathcal G,L_2(P_N)) \le m\log(A/\varepsilon), \qquad \forall \varepsilon\in(0,1],
\end{equation}
for a constant $A\ge 1$. Then the empirical Rademacher complexity of $\mathcal{G}$ on the sample $S$ satisfies
\begin{equation} \footnotesize 
    \widehat{\mathfrak R}_S(\mathcal G) \le \frac{6A\sqrt{\pi m}}{\sqrt N}.
\end{equation}
\end{lemma}

\begin{proof}
By Dudley's entropy integral bound~\cite{dudley1967sizes,dudley1987universal},
\begin{equation} \footnotesize
    \widehat{\mathfrak R}_S(\mathcal G) \le \inf_{\alpha>0} \left[4\alpha+\frac{12}{\sqrt N}\int_\alpha^1\sqrt{\log\mathcal N(\varepsilon,\mathcal G,L_2(P_N))}\,d\varepsilon \right].
\end{equation}

Letting $\alpha\downarrow 0$ yields
\begin{equation} \footnotesize 
    \widehat{\mathfrak R}_S(\mathcal G)\le\frac{12}{\sqrt N}\int_0^1\sqrt{\log \mathcal N(\varepsilon,\mathcal G,L_2(P_N))}\,d\varepsilon.
\end{equation}

Substituting the covering-number assumption gives
\begin{equation} \footnotesize 
    \widehat{\mathfrak R}_S(\mathcal G)\le\frac{12\sqrt m}{\sqrt N}\int_0^1\sqrt{\log(A/\varepsilon)}\,d\varepsilon.
\end{equation}

By the change of variables $u=\log(A/\varepsilon)$, which yields $\varepsilon=Ae^{-u}$ and $d\varepsilon=-Ae^{-u}du$, we obtain
\begin{equation} \footnotesize 
\begin{aligned}
    \int_0^1 \sqrt{\log(A/\varepsilon)}\,d\varepsilon&=A\int_{\log A}^{\infty}\sqrt u\,e^{-u}\,du \\&\le A\int_0^\infty \sqrt u\,e^{-u}\,du=A\,\Gamma\!\left(\frac32\right)=\frac{A\sqrt\pi}{2}.
\end{aligned}
\end{equation}

Therefore,
\begin{equation} \footnotesize 
    \widehat{\mathfrak R}_S(\mathcal G)\le \frac{12\sqrt m}{\sqrt N}\cdot \frac{A\sqrt\pi}{2}=\frac{6A\sqrt{\pi m}}{\sqrt N}.
\end{equation}
\end{proof}

\subsubsection{Risks and generalization comparison}
\label{app:risk_and_main_result}

\begin{definition}[Risks]
\label{def:risk}
Define the branch-wise risks
\begin{equation} \footnotesize 
\begin{aligned}
    R_{\mathrm{dir}}^L(f_L)&:=\mathbb{E}\bigl[\ell_L(f_L(l_t,h_t),l)\bigr],\\ R_{\mathrm{dir}}^H(f_H)&:=\mathbb{E}\bigl[\ell_H(f_H(l_t,h_t),h)\bigr], \\ R_{\mathrm{dec}}^L(g_L)&:=\mathbb{E}\bigl[\ell_L(g_L(l_t),l)\bigr],
    \\R_{\mathrm{dec}}^H(g_H)&:= \mathbb{E}\bigl[\ell_H(g_H(l,h_t),h)\bigr].
\end{aligned}
\end{equation}

The total risks are
\begin{equation} \footnotesize 
\begin{aligned}
    R_{\mathrm{dir}}(f) &:= R_{\mathrm{dir}}^L(f_L)+R_{\mathrm{dir}}^H(f_H), \\R_{\mathrm{dec}}(g_L,g_H) &:=  R_{\mathrm{dec}}^L(g_L)+R_{\mathrm{dec}}^H(g_H).
    \end{aligned}
\end{equation}

Their empirical counterparts, denoted by $\widehat R_{\mathrm{dir}}$ and $\widehat R_{\mathrm{dec}}$, are defined analogously.
\end{definition}

\begin{proposition}[Generalization comparison for explicit decoupling under simplified assumptions]
Let the ambient dimension be $d:=d_L+d_H$, and let $R_{\mathrm{dir}},\widehat{R}_{\mathrm{dir}}$ and $R_{\mathrm{dec}},\widehat{R}_{\mathrm{dec}}$ denote the corresponding true and empirical risks, respectively (see Definition~\ref{def:risk}). The following bounds hold simultaneously for all $f\in\mathcal{F}_{\mathrm{dir}}$ and $(g_L,g_H)\in\mathcal{F}_{\mathrm{dec}}$ with probability at least $1-\delta$:
\begin{equation}
\footnotesize
\begin{aligned}
R_{\mathrm{dir}}(f)
&\le
\widehat{R}_{\mathrm{dir}}(f)
+
\frac{24A\sqrt{\pi d}}{\sqrt{N}}
+
3\sqrt{\frac{2\log(8/\delta)}{N}},
\\[0.5em]
R_{\mathrm{dec}}(g_L,g_H)
&\le
\widehat{R}_{\mathrm{dec}}(g_l,g_h)
+
\frac{12A\sqrt{\pi}}{\sqrt{N}}
\bigl(\sqrt{d_l}+\sqrt{k_l+d_h}\bigr)
\\
&\quad
+
3\sqrt{\frac{2\log(8/\delta)}{N}}.
\end{aligned}
\label{eq:bound_dec_main1}
\end{equation}
where $k_L<d_L$ is the intrinsic dimension of the clean low-frequency component. Consequently, since $d=d_L+d_H$ and $k_L<d_L$, the decoupled complexity term is smaller than the corresponding direct-model term.
\end{proposition}

\begin{proof}
We first bound the Rademacher complexities of the four loss classes. By Assumption~\ref{ass:covering} and Lemma~\ref{lem:entropy_to_rademacher},
\begin{equation} \footnotesize 
\begin{aligned}
    \widehat{\mathfrak R}_S(\mathcal G_{\mathrm{dir}}^L) &\le \frac{6A\sqrt{\pi d}}{\sqrt N}, & \widehat{\mathfrak R}_S(\mathcal G_{\mathrm{dir}}^H) &\le \frac{6A\sqrt{\pi d}}{\sqrt N}, \\ \widehat{\mathfrak R}_S(\mathcal G_{\mathrm{dec}}^L) &\le \frac{6A\sqrt{\pi d_L}}{\sqrt N}, & \widehat{\mathfrak R}_S(\mathcal G_{\mathrm{dec}}^H) &\le \frac{6A\sqrt{\pi (k_L+d_H)}}{\sqrt N}.
\end{aligned}
\label{eq:rad_branches_final}
\end{equation}

In view of Assumption~\ref{ass:boundedness}, all losses are $[0,1]$-valued; thus, the standard uniform Rademacher generalization bound~\cite{mohri2018foundations} implies that for any class $\mathcal G$ of $[0,1]$-valued losses, with probability at least $1-\eta$,
\begin{equation} \footnotesize 
    R(g) \le \widehat R(g) + 2\widehat{\mathfrak R}_S(\mathcal G) + 3\sqrt{\frac{\log(2/\eta)}{2N}} \qquad \forall g\in\mathcal G.
\end{equation}

\paragraph{Direct model.}
Applying the above bound to $\mathcal G_{\mathrm{dir}}^L$ and $\mathcal G_{\mathrm{dir}}^H$ with confidence level $\eta=\delta/4$ for each class, a union bound yields that, with probability at least $1-\delta/2$, both inequalities hold simultaneously for all $f=(f_L,f_H)\in\mathcal F_{\mathrm{dir}}$:
\begin{equation} \footnotesize 
\begin{aligned}
    R_{\mathrm{dir}}^L(f_L) &\le \widehat R_{\mathrm{dir}}^L(f_L) + 2\widehat{\mathfrak R}_S(\mathcal G_{\mathrm{dir}}^L) + 3\sqrt{\frac{\log(8/\delta)}{2N}}, \\ R_{\mathrm{dir}}^H(f_H) &\le \widehat R_{\mathrm{dir}}^H(f_H) + 2\widehat{\mathfrak R}_S(\mathcal G_{\mathrm{dir}}^H) + 3\sqrt{\frac{\log(8/\delta)}{2N}}.
\end{aligned}
\end{equation}

Using Eq.~\eqref{eq:rad_branches_final} and summing the two branch-wise bounds yields
\begin{equation} \footnotesize
    R_{\mathrm{dir}}(f) \le \widehat R_{\mathrm{dir}}(f) + \frac{24A\sqrt{\pi d}}{\sqrt N} + 3\sqrt{\frac{2\log(8/\delta)}{N}}.
\end{equation}

\paragraph{Decoupled model.}
Applying the same argument to $\mathcal G_{\mathrm{dec}}^L$ and $\mathcal G_{\mathrm{dec}}^H$, again with confidence level $\eta=\delta/4$ for each class, a union bound yields that, with probability at least $1-\delta/2$, the following hold simultaneously for all $(g_L,g_H)\in\mathcal F_{\mathrm{dec}}$:
\begin{equation} \footnotesize
\begin{aligned}
    R_{\mathrm{dec}}^L(g_L) &\le \widehat R_{\mathrm{dec}}^L(g_L) + 2\widehat{\mathfrak R}_S(\mathcal G_{\mathrm{dec}}^L) + 3\sqrt{\frac{\log(8/\delta)}{2N}}, \\ R_{\mathrm{dec}}^H(g_H) &\le \widehat R_{\mathrm{dec}}^H(g_H) + 2\widehat{\mathfrak R}_S(\mathcal G_{\mathrm{dec}}^H) + 3\sqrt{\frac{\log(8/\delta)}{2N}}.
\end{aligned}
\end{equation}

Using Eq.~\eqref{eq:rad_branches_final} and summing gives
\begin{equation}
\footnotesize
\begin{aligned}
R_{\mathrm{dec}}(g_L,g_H)
&\le
\widehat R_{\mathrm{dec}}(g_L,g_H)
+
\frac{12A\sqrt{\pi}}{\sqrt N}
\bigl(\sqrt{d_L}+\sqrt{k_L+d_H}\bigr)
\\
&\quad
+
3\sqrt{\frac{2\log(8/\delta)}{N}}.
\end{aligned}
\end{equation}

\paragraph{Complexity comparison.}
Each of the two events above occurs with probability at least $1-\delta/2$. A final union bound implies that both inequalities hold simultaneously with probability at least $1-\delta$. From $k_L<d_L$, it follows that
\begin{equation} \footnotesize
    \sqrt{d_L}+\sqrt{k_L+d_H} < \sqrt{d_L}+\sqrt{d_L+d_H} < 2\sqrt{d_L+d_H} = 2\sqrt d.
\end{equation}

Multiplying both sides by $12A\sqrt\pi/\sqrt N$ establishes that the decoupled complexity term is strictly smaller than the corresponding direct-model term.
\end{proof}

\subsubsection{Error propagation and the practical architecture}
\label{app:plugin_result}

\begin{lemma}[Conditional error propagation]
\label{lem:plug_in}
Under Assumptions~\ref{ass:boundedness} and~\ref{ass:lipschitz}, for any $g_H\in\mathcal F_{\mathrm{dec}}^H$ and any $(l,\hat l,h_t,h)$, one has
\begin{equation} \footnotesize
    \ell_H\bigl(g_H(\hat l,h_t),h\bigr) \le \ell_H\bigl(g_H(l,h_t),h\bigr) + \frac{L_{\mathrm{cond}}}{B\sqrt{d_H}}\|\hat l-l\|_2.
\end{equation}
\end{lemma}

\begin{proof}
Fix $u:=g_H(\hat l,h_t)$ and $v:=g_H(l,h_t)$. By the definition of the normalized high-frequency loss,
\begin{equation} \footnotesize
    \ell_H(u,h)-\ell_H(v,h) = \frac{1}{4B^2d_H} \Bigl(\|u-h\|_2^2-\|v-h\|_2^2\Bigr).
\end{equation}

Using the identity $\|a\|_2^2-\|b\|_2^2=(a-b)^\top(a+b)$ with $a=u-h$ and $b=v-h$, we obtain
\begin{equation} \footnotesize
    \ell_H(u,h)-\ell_H(v,h)=\frac{1}{4B^2d_H}(u-v)^\top(u+v-2h).
\end{equation}

Taking absolute values and applying the Cauchy--Schwarz inequality yields
\begin{equation} \footnotesize
    \bigl|\ell_H(u,h)-\ell_H(v,h)\bigr| \le \frac{1}{4B^2d_H}\|u-v\|_2\,\|u+v-2h\|_2.
\end{equation}

By Assumption~\ref{ass:boundedness}, $\|u\|_\infty,\ \|v\|_\infty,\ \|h\|_\infty \le B$, so each coordinate of $u+v-2h$ has magnitude at most $4B$, which implies $\|u+v-2h\|_2 \le 4B\sqrt{d_H}$. 

Moreover, by Assumption~\ref{ass:lipschitz},
\begin{equation} \footnotesize
    \|u-v\|_2 = \|g_H(\hat l,h_t)-g_H(l,h_t)\|_2 \le L_{\mathrm{cond}}\|\hat l-l\|_2.
\end{equation}

Combining the two estimates yields
\begin{equation} \footnotesize 
    \bigl|\ell_H(u,h)-\ell_H(v,h)\bigr| \le \frac{L_{\mathrm{cond}}}{B\sqrt{d_H}}\|\hat l-l\|_2.
\end{equation}

The desired one-sided inequality follows immediately.
\end{proof}

\begin{corollary}[Generalization bound for the practical decoupled model]
\label{cor:plug_in}
Fix any low-frequency predictor $g_L\in\mathcal F_{\mathrm{dec}}^L$ independent of the randomness of the current sample, and let $(g_L,g_H)\in\mathcal F_{\mathrm{real}}$. Define the practical risk by $R_{\mathrm{real}}(g_L,g_H) := \mathbb{E}\!\left[ \ell_L\bigl(g_L(l_t),l\bigr) + \ell_H\bigl(g_H(g_L(l_t),h_t),h\bigr) \right]$. Then, under Assumptions~\ref{ass:boundedness}--\ref{ass:lipschitz}, with probability at least $1-\delta$, the following holds simultaneously for all $g_H\in\mathcal F_{\mathrm{dec}}^H$:
\begin{equation}
\footnotesize
\begin{aligned}
R_{\rm real}(g_L,g_H)
\le\;& \widehat{R}_{\rm dec}(g_L,g_H) \\
&+ \frac{12A\sqrt{\pi}}{\sqrt{N}}
\left(\sqrt{d_L}+\sqrt{k_L+d_H}\right) \\
&+ \frac{L_{\rm cond}}{B\sqrt{d_H}}
\mathbb{E}\left\|g_L(l_t)-l\right\|_2 \\
&+ 3\sqrt{\frac{2\log(8/\delta)}{N}} .
\end{aligned}
\label{eq:bound_real_final}
\end{equation}

\end{corollary}

\begin{proof}

Define the practical high-frequency risk
\begin{equation} \footnotesize
    R_{\mathrm{real}}^H(g_L,g_H):=\mathbb{E}\!\left[\ell_H\bigl(g_H(g_L(l_t),h_t),h\bigr)\right].
\end{equation}

By Lemma~\ref{lem:plug_in}, applied pointwise with $\hat l=g_L(l_t)$ and then averaged over the data distribution,
\begin{equation} \footnotesize
    R_{\mathrm{real}}^H(g_L,g_H)\le R_{\mathrm{dec}}^H(g_H) + \frac{L_{\mathrm{cond}}}{B\sqrt{d_H}} \mathbb{E}\|g_L(l_t)-l\|_2.
\end{equation}

Therefore,
\begin{equation}
\footnotesize
\begin{aligned}
R_{\mathrm{real}}(g_L,g_H)
&= R_{\mathrm{dec}}^L(g_L)+R_{\mathrm{real}}^H(g_L,g_H) \\
&\le R_{\mathrm{dec}}^L(g_L)+R_{\mathrm{dec}}^H(g_H) \\
&\quad + \frac{L_{\mathrm{cond}}}{B\sqrt{d_H}}
   \mathbb{E}\|g_L(l_t)-l\|_2 \\
&= R_{\mathrm{dec}}(g_L,g_H)
   + \frac{L_{\mathrm{cond}}}{B\sqrt{d_H}}
   \mathbb{E}\|g_L(l_t)-l\|_2 .
\end{aligned}
\end{equation}

Substituting the bound for $R_{\mathrm{dec}}(g_L,g_H)$ from Proposition~\ref{prop:decoupling_complexity} completes the proof.
\end{proof}

\begin{remark}
In the practical architecture, the high-frequency predictor conditions on the predicted $\hat{l}$ rather than the ground-truth $l$. Corollary~\ref{cor:plug_in} establishes that this modification introduces only an additional term controlled by the low-frequency prediction error. Consequently, the complexity advantage of explicit decoupling is retained provided that the structure predictor is sufficiently accurate.
\end{remark}

\subsection{Proofs of Theorem~3}\label{app:proof3}
\begin{proof}

We prove the two claims in turn.

Recall that
\[
\footnotesize
\mathbf{M}(t)=\lambda_l(t)\,\mathcal{W}_l^\top \mathcal{W}_l+\lambda_h(t)\,\mathcal{W}_h^\top \mathcal{W}_h.
\]

By the orthonormality of $\mathcal{W}$, we have $\mathcal{W}_l^\top \mathcal{W}_l+\mathcal{W}_h^\top \mathcal{W}_h=I_d$, which implies that for any nonzero $a\in\mathbb{R}^d$,
\begin{equation} \footnotesize
\begin{aligned}
    a^\top \mathbf{M}(t)a
    &=\lambda_l(t)\|\mathcal{W}_l a\|_2^2+\lambda_h(t)\|\mathcal{W}_h a\|_2^2
    \\&\ge \min\{\lambda_l(t),\lambda_h(t)\}\,\|a\|_2^2 > 0,
\end{aligned}
\end{equation}
where the equality uses $\|\mathcal{W}_l a\|_2^2+\|\mathcal{W}_h a\|_2^2=\|a\|_2^2$ and the last inequality follows from $\lambda_l(t),\lambda_h(t)>0$. Thus $\mathbf{M}(t)$ is symmetric positive definite for every $t\in[0,1]$.

For fixed $t$, $\epsilon$, and $z$, define $\delta := v_\theta(t,z)-v_t(z\mid x)$. Since $\mathbf{M}(t)$ is symmetric, expanding the quadratic form and collecting the $x$-independent term into $C_1$ gives
\begin{equation}
\footnotesize
\begin{aligned}
\mathcal{L}(\theta)
&= \int_0^1 \mathbb{E}_{t,x,\epsilon}\Bigl[
    v_\theta(t,x_t)^\top \mathbf{M}(t)v_\theta(t,x_t) \\
&\quad
    - 2v_\theta(t,x_t)^\top \mathbf{M}(t)v_t(x_t\mid x) \\
&\quad
    + v_t(x_t\mid x)^\top \mathbf{M}(t)v_t(x_t\mid x)
\Bigr]dt \\
&= \int_0^1 \mathbb{E}_{t,x,\epsilon}\Bigl[
    v_\theta(t,x_t)^\top \mathbf{M}(t)v_\theta(t,x_t) \\
&\quad
    - 2v_\theta(t,x_t)^\top \mathbf{M}(t)v_t(x_t\mid x)
\Bigr]dt + C_1 .
\end{aligned}
\label{eq:weighted_expansion}
\end{equation}
where $C_1:=\int_0^1\mathbb{E}_{t,x,\epsilon}\bigl[v_t(x_t\mid x)^\top \mathbf{M}(t) v_t(x_t\mid x)\bigr]dt$ is independent of $\theta$.

We now simplify the cross term. For each fixed $t$, expanding the expectation and using the fact that neither $v_\theta(t,z)$ nor $\mathbf{M}(t)$ depends on the conditioning variable $x$ yields
\begin{equation} \footnotesize
\begin{aligned}
    &\mathbb{E}_{t,x,\epsilon}\bigl[v_\theta(t,x_t)^\top \mathbf{M}(t) v_t(x_t\mid x)\bigr] \\
    &= \int_{\mathbb{R}^d} v_\theta(t,z)^\top \mathbf{M}(t)
       \Bigl(\int_{\mathbb{R}^d} v_t(z\mid x)\,p_t(z\mid x)\,\rho_1(x)\,dx\Bigr)dz \\[2pt]
    &= \int_{\mathbb{R}^d} v_\theta(t,z)^\top \mathbf{M}(t)\,v_t(z)\,p_t(z)\,dz
     \\&= \mathbb{E}_{z\sim p_t}\bigl[v_\theta(t,z)^\top \mathbf{M}(t) v_t(z)\bigr],
\end{aligned}
\label{eq:exp_of_vMv}
\end{equation}
where the second equality follows from the definition of the marginal velocity field $v_t(z)=\int v_t(z| x)p_t(z| x)\rho_1(x)\,dx$.

Substituting Eq.~\eqref{eq:exp_of_vMv} into Eq.~\eqref{eq:weighted_expansion} and completing the square yields
\begin{equation}
\footnotesize
\begin{aligned}
\mathcal{L}(\theta)
&= \int_0^1 \mathbb{E}_{z\sim p_t}\Bigl[
    \bigl(v_\theta(t,z)-v_t(z)\bigr)^\top \mathbf{M}(t) \\
&\quad
    \cdot \bigl(v_\theta(t,z)-v_t(z)\bigr)
    - v_t(z)^\top \mathbf{M}(t)v_t(z)
\Bigr]dt + C_1 \\
&= \int_0^1 \mathbb{E}_{z\sim p_t}\Bigl[
    \bigl(v_\theta(t,z)-v_t(z)\bigr)^\top \mathbf{M}(t) \\
&\quad
    \cdot \bigl(v_\theta(t,z)-v_t(z)\bigr)
\Bigr]dt + C .
\end{aligned}
\end{equation}
where $C:=C_1-\int_0^1\mathbb{E}_{z\sim p_t}\bigl[v_t(z)^\top \mathbf{M}(t)v_t(z)\bigr]dt$ is independent of $\theta$. 

By Step 1, $\mathbf{M}(t)$ is positive definite for every $t$. Hence, for any vector $a \in \mathbb{R}^d$, $a^\top \mathbf{M}(t)a\ge 0$, with equality if and only if $a=0$. Therefore, the integrand in Eq.~\eqref{eq:weighted_expansion} is nonnegative almost surely, and it vanishes if and only if
\begin{equation} \footnotesize
    v_\theta(t,x_t)=v_t(x_t).
\end{equation}

It follows that the weighted objective is minimized if and only if
\begin{equation} \footnotesize
    v_\theta(t,x_t)=v_t(x_t) 
\end{equation}
for Lebesgue-a.e. $t\in[0,1]$ and  $p_t\text{-a.e. } x_t\in\mathbb R^d$. Thus the unique minimizer of $\mathcal{L}(\theta)$, up to almost-everywhere equality, is
\begin{equation} \footnotesize
    v_\theta^*(t,x_t)=v_t(x_t).
\end{equation}

This completes the proof.
\end{proof}

\section{Experimental Details}\label{app:exp_details}
\subsection{Model Configuration}
To start, all experiments are conducted on a node with 8×A800 GPUs. The experiment configurations of our model are summarized in Table~\ref{tab:exp_configs}. In practice, we follow the training setups from previous works such as DiT~\cite{peebles2023scalable} and SiT~\cite{ma2024sit}. Notably, existing methods utilize a patch size of 16. In our framework, the low- and high-frequency predictors operate on sub-states derived via DWT, which possess spatial dimensions of $H/2 \times W/2 \times C$. To maintain scale consistency with these approaches, we consequently employ a patch size of 8. For the frequency decomposition, we use a single-level orthonormal Haar DWT. For an input of shape $H\times W \times C$, the low-frequency component (LL) has shape $H/2\times W/2 \times C$, while the high-frequency component is formed by concatenating the three detail sub-bands (LH, HL, HH) and has shape $H/2\times W/2\times 3C$.
\subsection{Detailed Architecture of \texttt{FREPix}}
\label{app:arch_details}

In this section, we provide a more detailed formulation of \texttt{FREPix}. Recall that at time $t$, the image state is decomposed by the orthonormal wavelet transform as
\begin{equation} \footnotesize
(l_t,h_t)=\mathcal{W}(x_t), \qquad x_t=\mathcal{W}^{-1}(l_t,h_t),
\end{equation}
where $l_t$ denotes the low-frequency sub-state and $h_t$ denotes the high-frequency sub-state. For a single-level 2D DWT applied to an input image of shape $H\times W\times C$, the low-frequency component has shape $H/2\times W/2 \times C$, while the high-frequency component is obtained by concatenating the three detail sub-bands and has shape $H/2\times W/2 \times 3C$.

\paragraph{Low-frequency DiT.}
Firstly, the low-frequency branch (DiT) tokenizes $l_t$ using non-overlapping patches of size $P\times P$. These patch vectors are projected into the DiT hidden space by a linear embedding layer $E_s(\cdot)$:
\begin{equation} \footnotesize
l_t^{\mathrm{tok}}=\mathrm{Unfold}(l_t)\in\mathbb{R}^{B\times L\times 3P^2},
\end{equation}
\begin{equation} \footnotesize
s_0 = E_s(l_t^{\mathrm{tok}})\in\mathbb{R}^{B\times L\times D},
\end{equation}
where $L=\frac{H/2}{P}\frac{W/2}{P}$ is the number of low-frequency patches. The condition vector $c$ combines the timestep embedding and the class embedding:
\begin{equation} \footnotesize
c=\mathrm{SiLU}\!\left(E_t(t)+E_y(y)\right)\in\mathbb{R}^{B\times 1\times D},
\end{equation}
where $E_t(\cdot)$ denotes the timestep embedder and $E_y(\cdot)$ denotes the label embedding layer. The low-frequency tokens are then processed by $K$ DiT blocks with 2D RoPE:
\begin{equation} \footnotesize
s_k = \mathrm{DiTBlock}_k(s_{k-1}, c, \mathrm{RoPE}),\qquad k=1,\dots,K.
\end{equation}

After the final block, the low-frequency tokens are projected back to the patch domain: 
\begin{equation} \footnotesize
l^{\mathrm{tok}} = W_l(s_k)\in\mathbb{R}^{B\times L\times 3P^2}.
\end{equation}

Finally, the clean low-frequency prediction is reconstructed by reshaping and folding these tokens back to the spatial grid:
\begin{equation} \footnotesize
\hat{l}=\mathrm{Fold}\big(\mathrm{Reshape}(l^{\mathrm{tok}})\big)\in\mathbb{R}^{B\times 3\times H/2\times W/2}.
\end{equation}

\paragraph{High-frequency decoder.}
The high-frequency branch follows a lightweight attention-free decoder from DeCo~\cite{ma2025deco}. We first patchify the high-frequency component $h_t$ and embed each patch with a linear layer $E_q(\cdot)$:
\begin{equation} \footnotesize
h_t^{\mathrm{tok}}=\mathrm{Unfold}(h_t)\in\mathbb{R}^{B\times L\times 9P^2},
\end{equation}
\begin{equation} \footnotesize
q_0=E_q(h_t^{\mathrm{tok}})\in\mathbb{R}^{(BL)\times P^2\times 9}.
\end{equation}

The decoder condition is constructed from both the final low-frequency semantic token $s_K$ and the predicted low-frequency patch token:
\begin{equation} \footnotesize
c' = \mathrm{Reshape}\!\Big(s_K + W_{\mathrm{s}}\big(\mathrm{sg}(l^{\mathrm{tok}})\big)\Big)\in\mathbb{R}^{(BL)\times D}.
\end{equation}

The decoder itself is a stack of patch-local residual MLP blocks:
\begin{equation}
\footnotesize
\begin{aligned}
q_m
&= q_{m-1}
 + \alpha_m(c')\odot \mathrm{MLP}_m\!\Big( \\
&\quad
    \gamma_m(c')\odot \mathrm{RMSNorm}(q_{m-1})
    + \beta_m(c')
 \Big),
\end{aligned}
\end{equation}
where $\alpha_m(\cdot)$, $\beta_m(\cdot)$ and $\gamma_m(\cdot)$ are AdaLN-Zero~\cite{peebles2023scalable} modulation parameters produced from the condition $c'$, respectively. After the final block, the decoder predicts the clean high-frequency patch tokens:
\begin{equation} \footnotesize
h^{\mathrm{tok}} = W_h(q_M) \in \mathbb{R}^{(BL)\times P^2\times 9}.
\end{equation}

The clean high-frequency prediction is reconstructed by reshaping and folding back to the spatial grid:
\begin{equation} \footnotesize
\hat{h}=\mathrm{Fold}\bigl(\mathrm{Reshape}(h^{\mathrm{tok}})\bigr)\in\mathbb{R}^{B\times 9\times H/2\times W/2}.
\end{equation}

\paragraph{Overall pipeline.}
The two predicted components are finally merged back into pixel space by the inverse DWT:
\begin{equation} \footnotesize
\hat{x}=\mathcal{W}^{-1}(\hat{l},\hat{h}).
\end{equation}
Therefore, the full generator can be written as
\begin{equation} \footnotesize
x_t
\;\xrightarrow{\;\mathcal{W}\;}\;
(l_t,h_t)
\;\xrightarrow{\;f_\varphi,\;g_\phi\;}\;
(\hat{l},\hat{h})
\;\xrightarrow{\;\mathcal{W}^{-1}\;}\;
\hat{x}.
\end{equation}

This architecture explicitly factorizes the prediction targets: the DiT predicts clean low-frequency structure first, and the decoder then predicts clean high-frequency detail conditioned on that structure. Since \texttt{FREPix} adopts an $x$-prediction parameterization, the reconstructed clean image $\hat{x}$ is subsequently converted into the induced velocity for flow-matching training.

\begin{table}[!t]
\centering
\scriptsize
\setlength{\tabcolsep}{2mm}
\renewcommand{\arraystretch}{0.8}
\begin{tabular}{l|cc}
\toprule
 & \texttt{FREPix-L} & \texttt{FREPix-XL}  \\
\midrule

\rowcolor[HTML]{F2F2F2}
\multicolumn{3}{l}{\textbf{architecture}} \\
DiT depth & 22 & 28 \\
hidden dim & 1024 & 1152 \\
heads & 16 & 16 \\
params & 420M & 674M \\
decoder depth & \multicolumn{2}{c}{3} \\
decoder hidden dim & \multicolumn{2}{c}{32} \\
patch size & \multicolumn{2}{c}{8} \\
dropout & 0.1 & 0.2 \\
image size & 256 & 256/512 \\
\midrule

\rowcolor[HTML]{F2F2F2}
\multicolumn{3}{l}{\textbf{representation alignment}~\cite{yure2024presentation}} \\
alignment depth & \multicolumn{2}{c}{8-th layer} \\
loss weight & \multicolumn{2}{c}{0.5} \\
alignment encoder & \multicolumn{2}{c}{Frozen DINOv2} \\
\midrule

\rowcolor[HTML]{F2F2F2}
\multicolumn{3}{l}{\textbf{perceptual supervision}~\cite{zhang2018unreasonable}} \\
loss weight & \multicolumn{2}{c}{0.5} \\
perceptual encoder & \multicolumn{2}{c}{Frozen VGG} \\
\midrule

\rowcolor[HTML]{F2F2F2}
\multicolumn{3}{l}{\textbf{training}} \\
optimizer & \multicolumn{2}{c}{AdamW ($\beta_1,\beta_2=0.9,0.999$)} \\
batch size & \multicolumn{2}{c}{256} \\
learning rate & \multicolumn{2}{c}{1e-4} \\
lr schedule & \multicolumn{2}{c}{constant} \\
weight decay & \multicolumn{2}{c}{0} \\
ema decay & \multicolumn{2}{c}{0.9999} \\
time sampler & \multicolumn{2}{c}{$\mathrm{logit}(t)\sim\mathcal{N}(-0.8,0.8^2)$} \\
noise scale & \multicolumn{2}{c}{1.0} \\
path smooth $\varepsilon$ & \multicolumn{2}{c}{0.01} \\
\midrule

\rowcolor[HTML]{F2F2F2}
\multicolumn{3}{l}{\textbf{sampling}} \\
ODE solver & \multicolumn{2}{c}{Euler} \\
ODE steps & 50 & 25,50,100 \\
timeshift & 1.0 & 2.0 \\
CFG scale & \multicolumn{2}{c}{3.0 ($256\times256$),
4.5 ($512\times512$)} \\
CFG interval~\cite{kynkaanniemi2024applying}
& \multicolumn{2}{c}{[0.15, 1]} \\
\bottomrule
\end{tabular}
\caption{Configurations of experiments.}
\label{tab:exp_configs}
\end{table}

\subsection{Baseline Comparison}
This subsection provides further implementation details for baseline comparison experiments. To ensure a fair comparison, all experiments are conducted using the L-sized model. For computational efficiency, we train the models at $256 \times 256$ resolution for 40 epochs (200k steps). The REPA loss is applied to all models except DiT-L/2 and PixelFlow following their settings. During inference, we use 50 steps Euler solver \textbf{without} CFG. The batch size and other hyperparameter follow the default settings in Table~\ref{tab:exp_configs}.

\begin{algorithm*}[!t]
\caption{Training step}
\label{alg:training}
\begin{algorithmic}[1]
\Require $\text{$f_\varphi$}$: low-frequency predictor; $\text{$g_\phi$}$: high-frequency predictor; $x$: training batch; $\lambda_l(t),\lambda_h(t)$: time-dependent weights.
\State $t = \text{sample\_t}()$ 
\State $\epsilon = \text{randn\_like}(x)$ \Comment{Sample Gaussian noise}
\State $(l, h)= \mathcal{W}(x) ,\quad (\epsilon_l, \epsilon_h)=\mathcal{W}(\epsilon)$ \Comment{DWT}
\State $l_t = g_l(t)\,l + \bigl(1-g_l(t)\bigr)\epsilon_l, \quad
h_t = g_h(t)\,h + \bigl(1-g_h(t)\bigr)\epsilon_h $ \Comment{Heterogeneous interpolation}
\State $v_t^l=\dot{g_l}(t)(l-\epsilon_l),\quad v_t^h=\dot{g_h}(t)(h-\epsilon_h)$ \Comment{Target velocity}
\State $\hat{l} = f_\varphi(l_t,t)$ \Comment{Predicted clean low-freq}
\State $\hat{h} = g_\phi(h_t,\hat{l},t)$ \Comment{Predicted clean high-freq}
\State $\hat{x}=\mathcal{W}^{-1}(\hat{l},\hat{h})$ \Comment{Predicted clean image}
\State $v_\theta^l = \frac{\dot g_l(t)}{1-g_l(t)}(\hat l-l_t), \qquad
v_\theta^h = \frac{\dot g_h(t)}{1-g_h(t)}(\hat h-h_t)$ \Comment{Predicted velocity}
\State $\mathcal{L}=\lambda_l(t)||v_{\theta}^l-v_{t}^l||^2+\lambda_h(t)||v_{\theta}^h-v_{t}^h||^2$ \Comment{Compute reweighted v-loss}
\State $\textbf{loss} \leftarrow \mathcal{L}+\mathcal{L}_{\mathrm{REPA}}+\mathcal{L}_{\mathrm{LPIPS}}$ \Comment{Loss}
\end{algorithmic}
\end{algorithm*}

\begin{algorithm*}[!t]
\caption{Sampling step (Euler)}
\label{alg:sampling}
\begin{algorithmic}[1]
\Require $x_t$: current samples at $t$; $t, t_{next}$.
\State $x_{pred} \leftarrow \text{net}_\theta(x_t, t)$ \Comment{Network prediction}
\State $v_{pred}=\dot{G}(t)\bigl(I-G(t)\bigr)^{-1}(\hat{x}-x_t)$ \Comment{Estimate velocity}
\State $x_{next} \leftarrow x_{pred} + (t_{next} - t) \cdot v_{pred}$ \Comment{Euler update}
\State \textbf{return} $x_{next}$
\end{algorithmic}
\end{algorithm*}
\subsection{Class-to-Image Generation}\label{app:main_details}
This subsection provides further implementation details for class-to-image generation. For ImageNet class-to-image experiments, we initially train the XL-sized model (\texttt{FREPix-XL}) at 256$\times$256 resolution for 320 epochs (1.6M steps), followed by fine-tuning at 512$\times$512 resolution for an additional 10 epochs (50k steps). During inference, we use 100 steps Euler solver incorporated with Classifier-Free Guidance (CFG) and a guidance interval. The batch size and learning rate follow the default settings in Table~\ref{tab:exp_configs}. We utilize a global batch size of 256 and the AdamW optimizer with a constant learning rate of $1 \times 10^{-4}$. The time sampler utilizes a logit-normal distribution over $t$: $\text{logit}(t) \sim \mathcal{N}(-0.8, 0.8^2)$, which aligns with JiT~\cite{li2025back}. We set the CFG scale to 3.0 for 256$\times$256 resolution (320 epochs) and 4.5 for 512$\times$512 resolution (totaling 330 epochs). We use the CFG guidance interval of $[0.15,1]$ for the default configuration.

\subsection{Ablation Study}\label{app:ablation_details}
This subsection provides additional implementation details for the ablation studies. All ablation experiments are conducted using the L-sized model (\texttt{FREPix-L}). For computational efficiency, we train the models at $256 \times 256$ resolution for 40 epochs (200k steps). During inference, we utilize a 50-step Euler solver \textbf{without} CFG. The batch size and learning rate follow the default settings described previously. For power exponent ablation studies, we set the reweighting strength to $\omega = 0.7$ (our final configuration). For ablation studies of reweighting strength, we employ power exponents of $\gamma_l = 0.95$ and $\gamma_h = 1.05$ (our final settings). To ensure a fair comparison, all models are trained with Large-sized and sampled using the same steps.

\section{Pseudo-codes for Training and Sampling}
In this section, we provide the detailed pseudo-codes for the training and sampling procedures of our proposed framework.

\section{Additional Experiments and Results}\label{app:more_result}
This section provides more experimental results and qualitative results. 

\subsection{Additional Experiments}\label{app:addition_exp}

\paragraph{CFG guidance scale and interval.}
We report the classifier-free guidance (CFG) settings used for \texttt{FREPix-XL} on ImageNet 256$\times$256. Table~\ref{tab:cfg_settings} lists the CFG scale, guidance interval, and the resulting gFID, Inception Score (IS), precision, and recall for models trained for 80 and 320 epochs. For the 80 epochs, the best FID of 2.29 is achieved with a relatively higher CFG value of 3.0 and an interval of $[0.15,\;1]$. For the 320 epochs, the best FID of 1.91 is achieved with a relatively higher CFG scale of 3.0 and an interval of $[0.15,\;1]$. Compared with other settings, this configuration slightly reduces IS and precision while improving recall.
\begin{table}[!t]
\centering
\small
\setlength{\tabcolsep}{1pt}
\begin{tabular}{cccc|cccc}
\toprule
\textbf{Steps} & \textbf{Epochs} & \textbf{CFG} & \textbf{CFG interval} & \textbf{FID$\downarrow$} &\textbf{IS$\uparrow$} & \textbf{Pre.$\uparrow$} & \textbf{Rec.$\uparrow$}\\
\midrule
400k & 80 & 2.75 & $[0.1,\;1.0]$ & 2.61 & 294.3 & 0.80 & 0.59\\
400k & 80 & 2.75 & $[0.15,\;1.0]$ & 2.35 & 283.8 & 0.79 & 0.60\\
400k & 80 & 3.00 & $[0.1,\;1.0]$ & 2.62 & 313.4 & 0.80 & 0.59\\
\rowcolor[HTML]{F2F2F2} 400k & 80 & 3.00 & $[0.15,\;1.0]$ & 2.29 & 294.9 & 0.79 & 0.60\\
\midrule
1600k & 320 & 2.75 & $[0.1,\;1.0]$ & 2.09 & 310.0 & 0.80 & 0.61\\
1600k & 320 & 2.75 & $[0.15,\;1.0]$ & 1.94 & 300.4 & 0.79 & 0.61\\
1600k & 320 & 3.00 & $[0.1,\;1.0]$ &  2.07 &  317.6 & 0.81 & 0.61\\
\rowcolor[HTML]{F2F2F2} 1600k & 320 & 3.00 & $[0.15,\;1.0]$ & 1.91 & 295.6 & 0.79 & 0.62 \\
\bottomrule
\end{tabular}
\caption{CFG settings and results for \texttt{FREPix-XL} at ImageNet256. For sampling, we use Euler solver with 100 steps.}
\label{tab:cfg_settings}
\end{table}

\subsection{Additional Qualitative Results}

We provide additional qualitative results to further assess the visual fidelity and frequency-decoupled generation behavior of \texttt{FREPix}. Fig.~\ref{fig:app_vis_1} to Fig.~\ref{fig:app_vis_8} shows uncurated ImageNet256 samples generated by \texttt{FREPix-XL} (settings: train epochs: 320, CFG value: 3.0). In addition to the final generated images, we visualize the corresponding low- and high-frequency components obtained by the same DWT used in our method. These results demonstrate that \texttt{FREPix} produces coherent global structures in the low-frequency branch while preserving localized details and textures in the high-frequency branch.

\begin{figure*}[t]
  \centering
  \begin{subfigure}{\textwidth}
    \centering
    \includegraphics[width=\textwidth]{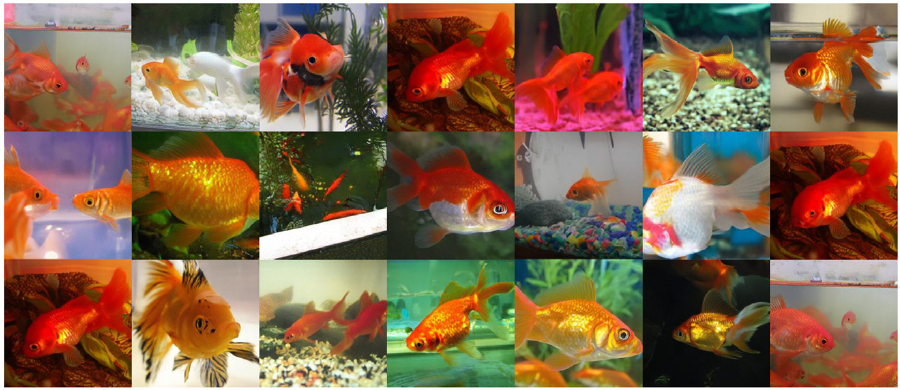}
    \caption{Uncurated class-conditional samples generated by \texttt{FREPix-XL}.}
    \label{fig:samples1}
  \end{subfigure}

  \begin{subfigure}{\textwidth}
    \centering
    \includegraphics[width=\textwidth]{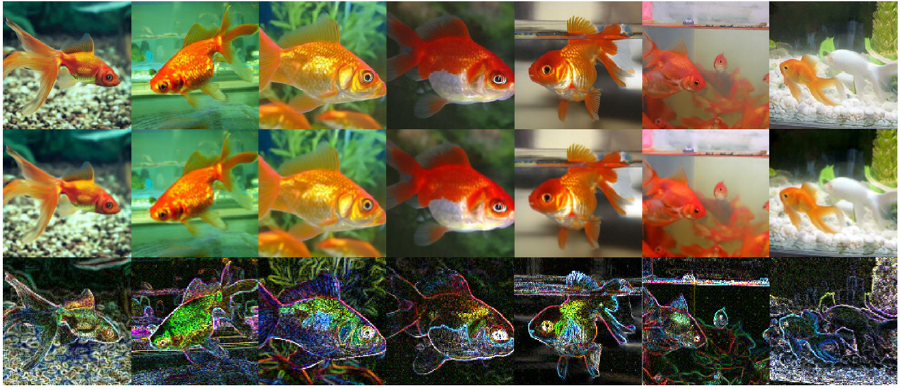}
    \caption{Frequency-decoupled visualization of the generated samples. From top to bottom: generated image, low-frequency component (LF), and high-frequency component (HF). For clarity, HF is displayed using coefficient magnitudes with logarithmic compression, and LF/HF are normalized separately.}
    \label{fig:freq_vis1}
  \end{subfigure}
  \caption{Uncurated samples generated by \texttt{FREPix-XL} conditioned on class 1: \textit{goldfish}.}
  \label{fig:app_vis_1}
\end{figure*}

\begin{figure*}[htbp]
  \centering
  \begin{subfigure}{\textwidth}
    \centering
    \includegraphics[width=\textwidth]{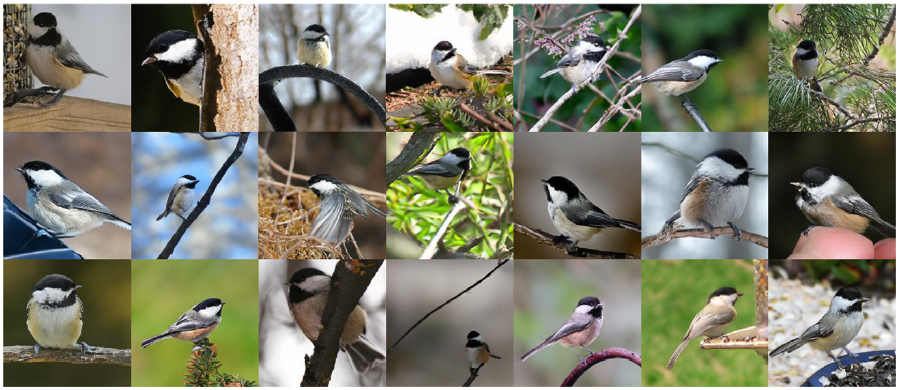}
    \caption{Uncurated class-conditional samples generated by \texttt{FREPix-XL}.}
    \label{fig:samples2}
  \end{subfigure}

  \begin{subfigure}{\textwidth}
    \centering
    \includegraphics[width=\textwidth]{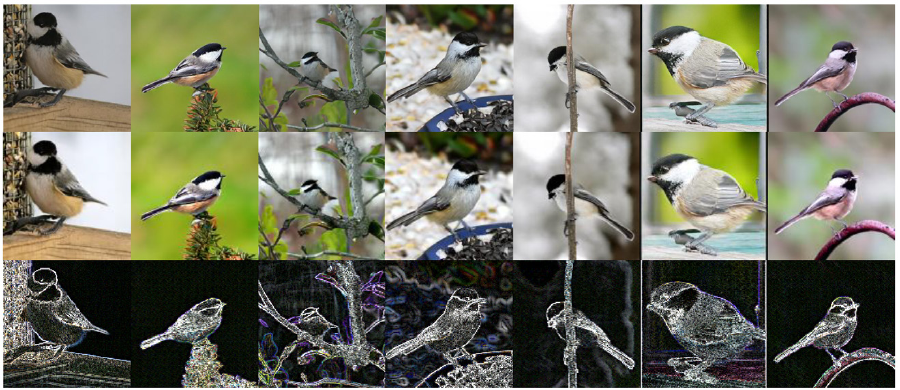}
    \caption{Frequency-decoupled visualization of the generated samples. From top to bottom: generated image, low-frequency component (LF), and high-frequency component (HF). For clarity, HF is displayed using coefficient magnitudes with logarithmic compression, and LF/HF are normalized separately.}
    \label{fig:freq_vis2}
  \end{subfigure}
  \caption{Uncurated samples generated by \texttt{FREPix-XL} conditioned on class 19: \textit{chickadee}.}
  \label{fig:app_vis_2}
\end{figure*}

\begin{figure*}[htbp]
  \centering
  \begin{subfigure}{\textwidth}
    \centering
    \includegraphics[width=\linewidth]{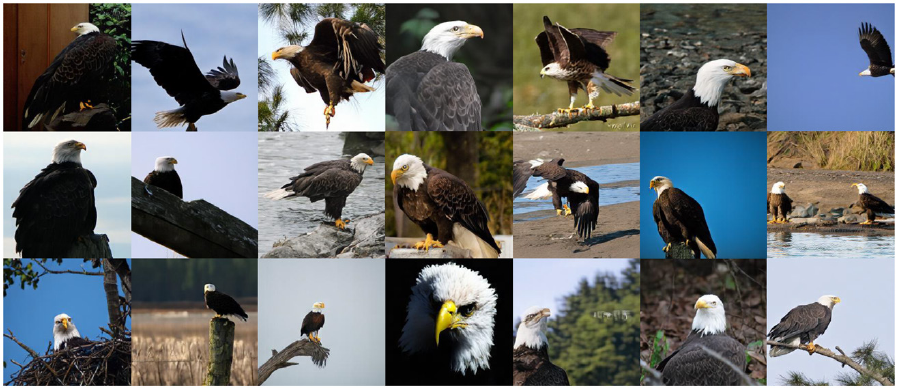}
    \caption{Uncurated class-conditional samples generated by \texttt{FREPix-XL}.}
    \label{fig:samples3}
  \end{subfigure}

  \begin{subfigure}{\textwidth}
    \centering
    \includegraphics[width=\linewidth]{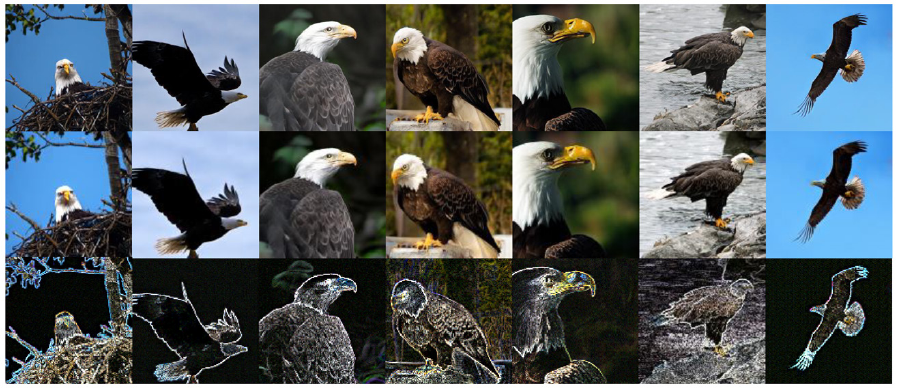}
    \caption{Frequency-decoupled visualization of the generated samples. From top to bottom: generated image, low-frequency component (LF), and high-frequency component (HF). For clarity, HF is displayed using coefficient magnitudes with logarithmic compression, and LF/HF are normalized separately.}
    \label{fig:freq_vis3}
  \end{subfigure}
  \caption{Uncurated samples generated by \texttt{FREPix-XL} conditioned on class 22: \textit{bald eagle}.}
  \label{fig:app_vis_3}
\end{figure*}

\begin{figure*}[htbp]
  \centering
  \begin{subfigure}{\textwidth}
    \centering
    \includegraphics[width=\linewidth]{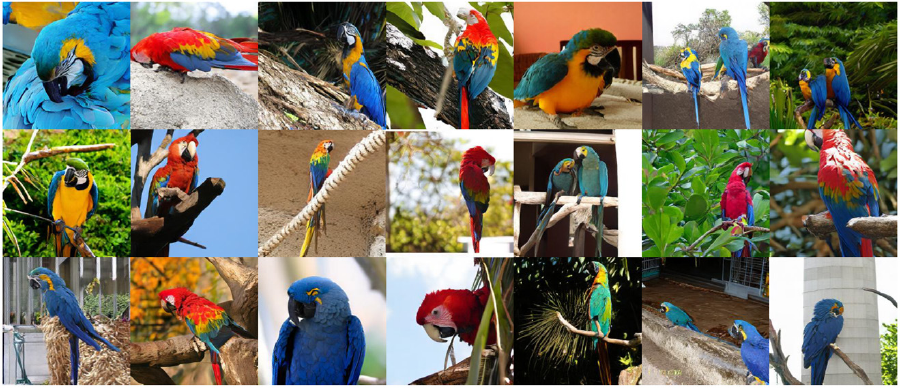}
    \caption{Uncurated class-conditional samples generated by \texttt{FREPix-XL}.}
    \label{fig:samples4}
  \end{subfigure}

  \begin{subfigure}{\textwidth}
    \centering
    \includegraphics[width=\linewidth]{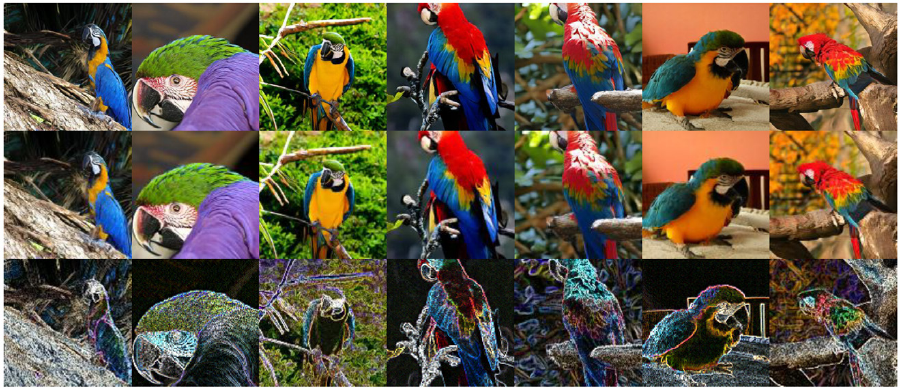}
    \caption{Frequency-decoupled visualization of the generated samples. From top to bottom: generated image, low-frequency component (LF), and high-frequency component (HF). For clarity, HF is displayed using coefficient magnitudes with logarithmic compression, and LF/HF are normalized separately.}
    \label{fig:freq_vis3}
  \end{subfigure}
  \caption{Uncurated samples generated by \texttt{FREPix-XL} conditioned on class 88: \textit{macaw}.}
  \label{fig:app_vis_4}
\end{figure*}

\begin{figure*}[htbp]
  \centering
  \begin{subfigure}{\textwidth}
    \centering
    \includegraphics[width=\linewidth]{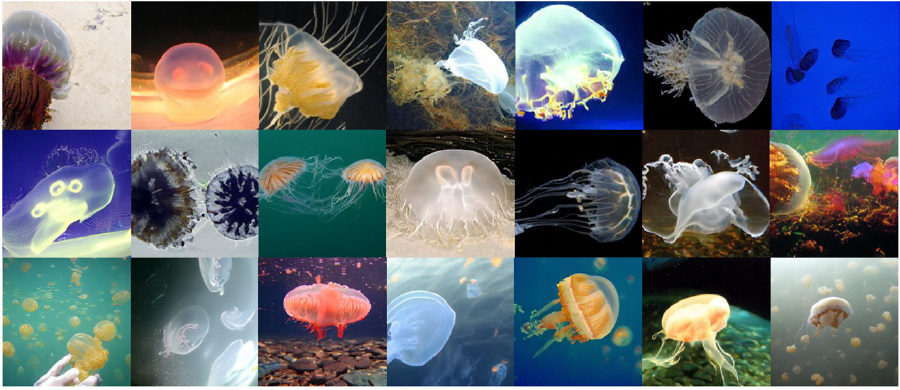}
    \caption{Uncurated class-conditional samples generated by \texttt{FREPix-XL}.}
    \label{fig:samples5}
  \end{subfigure}

  \begin{subfigure}{\textwidth}
    \centering
    \includegraphics[width=\linewidth]{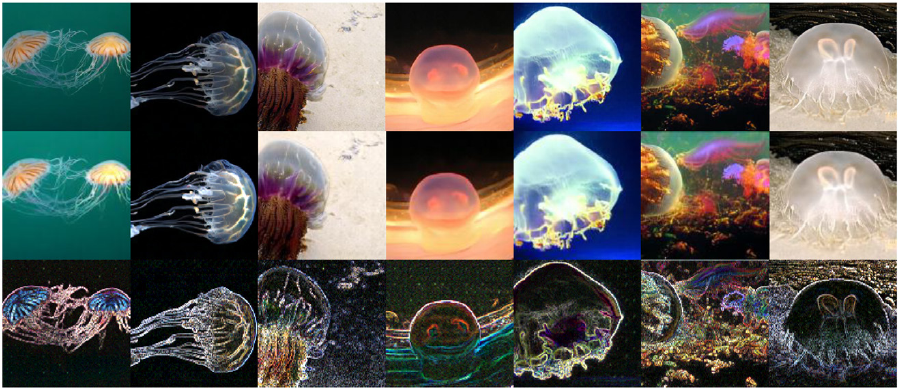}
    \caption{Frequency-decoupled visualization of the generated samples. From top to bottom: generated image, low-frequency component (LF), and high-frequency component (HF). For clarity, HF is displayed using coefficient magnitudes with logarithmic compression, and LF/HF are normalized separately.}
    \label{fig:freq_vis5}
  \end{subfigure}
  \caption{Uncurated samples generated by \texttt{FREPix-XL} conditioned on class 107: \textit{jellyfish}.}
  \label{fig:app_vis_5}
\end{figure*}

\begin{figure*}[htbp]
  \centering
  \begin{subfigure}{\textwidth}
    \centering
    \includegraphics[width=\linewidth]{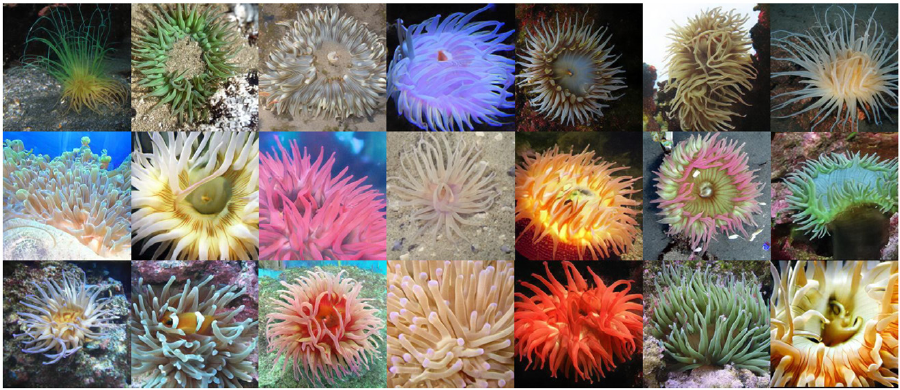}
    \caption{Uncurated class-conditional samples generated by \texttt{FREPix-XL}.}
    \label{fig:samples6}
  \end{subfigure}

  \begin{subfigure}{\textwidth}
    \centering
    \includegraphics[width=\linewidth]{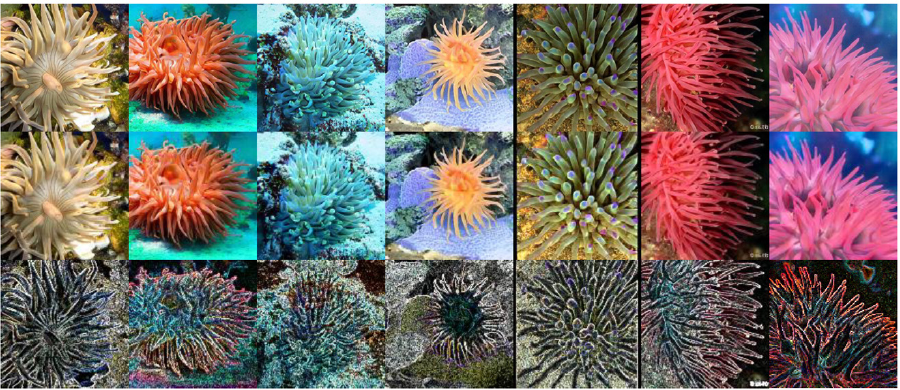}
    \caption{Frequency-decoupled visualization of the generated samples. From top to bottom: generated image, low-frequency component (LF), and high-frequency component (HF). For clarity, HF is displayed using coefficient magnitudes with logarithmic compression, and LF/HF are normalized separately.}
    \label{fig:freq_vis6}
  \end{subfigure}
  \caption{Uncurated samples generated by \texttt{FREPix-XL} conditioned on class 108: \textit{sea anemone}.}
  \label{fig:app_vis_6}
\end{figure*}

\begin{figure}[htbp]
  \centering
  \begin{subfigure}{\textwidth}
    \centering
    \includegraphics[width=\linewidth]{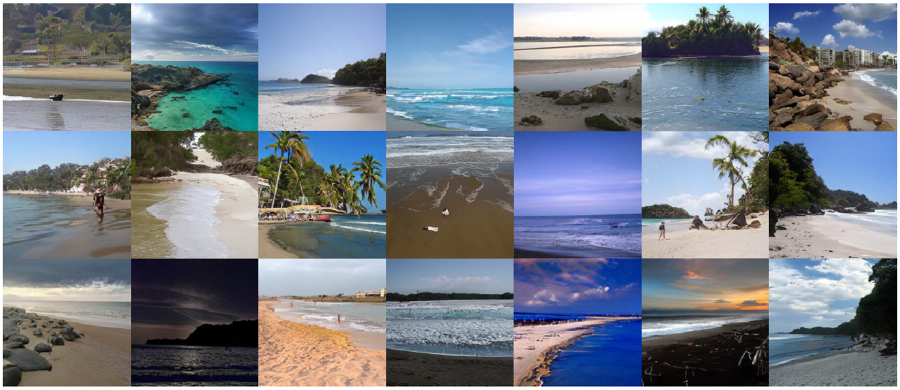}
    \caption{Uncurated class-conditional samples generated by \texttt{FREPix-XL}.}
    \label{fig:samples7}
  \end{subfigure}

  \begin{subfigure}{\textwidth}
    \centering
    \includegraphics[width=\linewidth]{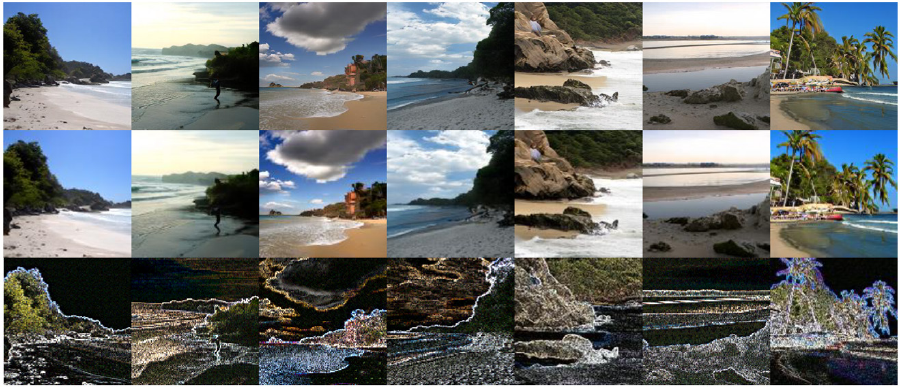}
    \caption{Frequency-decoupled visualization of the generated samples. From top to bottom: generated image, low-frequency component (LF), and high-frequency component (HF). For clarity, HF is displayed using coefficient magnitudes with logarithmic compression, and LF/HF are normalized separately.}
    \label{fig:freq_vis7}
  \end{subfigure}
  \caption{Uncurated samples generated by \texttt{FREPix-XL} conditioned on class 978: \textit{seashore}.}
  \label{fig:app_vis_7}
\end{figure}

\begin{figure*}[htbp]
  \centering
  \begin{subfigure}{\textwidth}
    \centering
    \includegraphics[width=\linewidth]{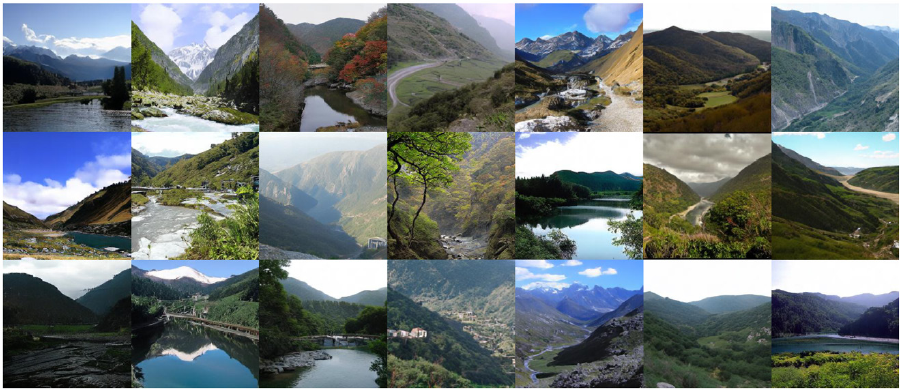}
    \caption{Uncurated class-conditional samples generated by \texttt{FREPix-XL}.}
    \label{fig:samples8}
  \end{subfigure}

  \begin{subfigure}{\textwidth}
    \centering
    \includegraphics[width=\linewidth]{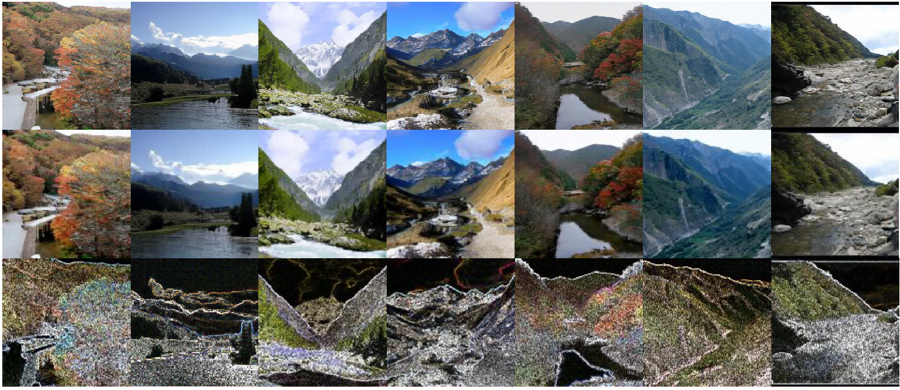}
    \caption{Frequency-decoupled visualization of the generated samples. From top to bottom: generated image, low-frequency component (LF), and high-frequency component (HF). For clarity, HF is displayed using coefficient magnitudes with logarithmic compression, and LF/HF are normalized separately.}
    \label{fig:freq_vis8}
  \end{subfigure}
  \caption{Uncurated samples generated by \texttt{FREPix-XL} conditioned on class 979: \textit{valley}.}
  \label{fig:app_vis_8}
\end{figure*}


\end{document}